\documentclass{bmvc2k}

\usepackage{booktabs,makecell,xcolor}
\definecolor{ForestGreen}{rgb}{0.13,0.55,0.13}
\usepackage{booktabs}      
\usepackage{multirow}      
\usepackage{makecell}      
\usepackage{graphicx}      
\usepackage{adjustbox}     
\usepackage{siunitx}       
\usepackage{booktabs}   
\usepackage{makecell}   
\usepackage{xcolor}     
\usepackage{amsmath}    
\usepackage{graphicx}   

\usepackage{booktabs}      
\usepackage{multirow}      
\usepackage{makecell}      
\usepackage{graphicx}      
\usepackage{adjustbox}     
\usepackage{siunitx}       
\usepackage{xcolor}        
\usepackage{amsmath}       
\usepackage{amssymb}       
\usepackage{amsthm}        
\usepackage{hyperref}      
\usepackage{cleveref}      
\usepackage{verbatim}
\usepackage{array}
\usepackage[table,dvipsnames]{xcolor}
\usepackage[table]{xcolor}
\newcommand{\best}[1]{\cellcolor{gray!25}\textbf{#1}}
\usepackage{wrapfig}
\newtheorem{theorem}{Theorem}
\newtheorem{proposition}[theorem]{Proposition}

\theoremstyle{definition}
\newtheorem{assumption}[theorem]{Assumption}
\theoremstyle{remark}

\newcommand{\Tfinal}{T_{\mathrm{final}}}
\newcommand{\Tearly}{T_{\mathrm{early}}}
\newcommand{\hfinal}{h_{\mathrm{final}}}
\newcommand{\hearly}{h_{\mathrm{early}}}
\newcommand{\hs}{h_{S}}
\newcommand{\hSbar}{\bar{h}_{S}}
\newcommand{\zs}{z_{S}}
\newcommand{\zf}{z_{F}}
\newcommand{\ze}{z_{E}}
\newcommand{\Ltc}{\mathcal{L}_{\mathrm{TC}}}
\newcommand{\Lss}{\mathcal{L}_{\mathrm{SS}}}
\newcommand{\Lkd}{\mathcal{L}_{\mathrm{KD}}}
\newcommand{\Lce}{\mathcal{L}_{\mathrm{CE}}}
\newcommand{\Lasd}{\mathcal{L}_{\mathrm{ASD}}}
\newcommand{\Csc}{\mathbf{C}_{\mathrm{sc}}}
\newcommand{\Cschat}{\widehat{\mathbf{C}}_{\mathrm{sc}}}
\newcommand{\Uk}{\mathbf{U}_{K}}

\newcommand{\Vrob}{\mathcal{R}}
\newcommand{\Vsc}{\mathcal{S}}
\newcommand{\Dh}{\Delta h}
\newcommand{\E}{\mathbb{E}}
\newcommand{\R}{\mathbb{R}}
\newcommand{\norm}[1]{\left\Vert #1 \right\Vert}

\definecolor{ForestGreen}{rgb}{0.13,0.55,0.13}

\providecommand{\mathcolor}[2]{\textcolor{#1}{#2}}

\title{Anti-Shortcut Distillation via Temporal \\
Negative Knowledge Transfer}

\addauthor{Syed Muhammad Raza}{iamraza1998@gmail.com}{1}
\addauthor{Omer Tariq}{omer.tariq@neubility.co.kr}{1}
\addauthor{Jeongbae Son}{thswjdqo92@neubility.co.kr}{1}

\addinstitution{
 Perception AI, \\ Neubility Inc.,  \\ Seoul, South Korea
}

\runninghead{Raza, Tariq, Son}{Anti-Shortcut Distillation}

\def\etal{\emph{et al}\bmvaOneDot}

\begin{document}

\maketitle

\vspace{-5mm}
\begin{abstract}
Knowledge distillation (KD) trains a compact student by attracting it towards a converged teacher. It is silent about which directions the teacher itself learned to suppress: repulsive and bias-aware objectives exist, but none exploits the teacher's own trajectory to identify what the student should avoid. We observe that the missing signal is already encoded in the teacher's optimization trajectory: features that an early-stage teacher emphasizes but that a converged teacher attenuates are precisely the shortcut directions worth pushing the student
away from. We instantiate this observation as \textbf{A}nti-\textbf{S}hortcut \textbf{D}istillation (ASD), a push--pull KD framework that treats the converged teacher $\Tfinal$ as a positive semantic anchor and an early-checkpoint teacher $\Tearly$ as a temporal negative reference. ASD couples two losses: a temporal contrastive loss ($\Ltc$) that places the early-teacher feature as a same-sample negative against in-batch and memory-bank final-teacher features in an InfoNCE objective; and a shortcut suppression loss ($\Lss$) that penalizes student projection onto the top eigenvectors of $\E[\Dh\Dh^{\top}]$, the uncentered second-moment matrix of early-to-final feature displacements. Across 13 teacher--student pairs on CIFAR-100, ImageNet-100, and TinyImageNet, ASD attains the highest clean top-1 accuracy on more than 10 pairs and outperforms standard KD on 12. On CIFAR-100-C corruption robustness, ASD obtains the lowest mean Corruption Error ($86.1$\,mCE) on the most challenging cross-architecture pair (WRN-40-2$\to$ShuffleNet-V2). Mechanistic diagnostics confirm the intended geometry: the ASD student is systematically anti-aligned with the shortcut direction, while its projection onto the robust subspace is substantially larger ($0.45$ vs.\ $0.12$). \href{https://github.com/SMRaza1009/ASD-Anti-Shortcut-Distillation-via-Temporal-Negative-Knowledge-Transfer-} {ASD github repository}
\end{abstract}

\vspace{-4mm}
\section{Introduction}
\label{sec:intro}

Knowledge distillation (KD) \cite{hinton2015distilling, gou2021knowledge}
compresses a high-capacity teacher into a compact student by transferring
supervisory structure richer than one-hot labels. Recent methods have expanded the types of knowledge transferred during distillation: FitNets \cite{romero2014fitnets} regresses intermediate features, AT~\cite{zagoruyko2016paying} matches spatial attention,
RKD~\cite{park2019relational} and PKT~\cite{passalis2018learning} preserve
sample geometry, CRD~\cite{tian2019contrastive} adds contrastive transfer,
DKD~\cite{zhao2022decoupled} decouples logits, and ReviewKD~\cite{chen2021distilling} aggregates feature hierarchies, but all share one geometric premise: distillation is treated as a
unilateral attractive force. The student is pulled toward the
teacher, but the objective is silent about which representational
directions the student should avoid. Repulsive sample-level negatives~\cite{tian2019contrastive}, checkpoint-based positive targets~\cite{jin2022efficient}, and bias-aware distillation~\cite{lukasik2021teacher,lee2023debiased} partially address this asymmetry, but none uses the teacher itself to identify harmful representational directions.

This asymmetry matters because teacher representations are themselves
products of a temporal optimization path. Networks tend to exploit local
textures, colour statistics, and background cues early in training before
consolidating shape-biased, semantically stable
features~\cite{arpit2017closer,geirhos2019imagenet,geirhos2020shortcut}. A converged teacher records the endpoint of this process, not the
shortcut-prone directions traversed and attenuated along the way. A compact student trained under the same data distribution may therefore recover the easiest spurious correlations available, weakening robustness under corruption and shift~\cite{hendrycks2019benchmarking,sagawa2020distributionally,shah2020pitfalls}. The teacher's trajectory thus carries a second form of knowledge: not only features worth imitating, but also directions worth rejecting.

We introduce Anti-Shortcut Distillation (ASD) to convert this
latent trajectory signal into explicit push--pull supervision
(\cref{fig:pipeline}). ASD treats the converged teacher $\Tfinal$ as a
positive semantic anchor and an early-checkpoint teacher $\Tearly$ as a
temporal negative reference. For an input $x$, the displacement
$\Dh(x) = \hearly(x) - \hfinal(x)$ is a sample-dependent proxy for the
representation components emphasized by the early state and attenuated by
convergence. This sharply separates ASD from prior trajectory-aware
distillation~\cite{jin2022efficient,mirzadeh2020improved,huang2017snapshot,
furlanello2018born,zhang2019your}: intermediate checkpoints are not
additional teachers to imitate, but evidence of representational
directions from which the student should be separated.

ASD realizes this temporal signal through two complementary objectives. A
temporal contrastive loss $\Ltc$ places the early-teacher feature as a
same-sample negative against in-batch and memory-queue final-teacher
features in an InfoNCE~\cite{oord2018representation,wu2018unsupervised,poole2019variational}
objective, so the student is pulled toward $\Tfinal$ and pushed away from
$\Tearly$ for the same input. A shortcut-suppression loss $\Lss$ penalizes
the student's projection onto the leading eigenvectors of the uncentered
second-moment matrix of early-to-final feature differences, constraining the
representation at the subspace level rather than per sample. Both auxiliary
terms are linearly warmed up and combined with cross-entropy and logit
distillation; a trainable projector handles mismatched feature dimensions
and is discarded at inference.

\begin{figure*}[t]
  \centering
  \includegraphics[width=\textwidth]{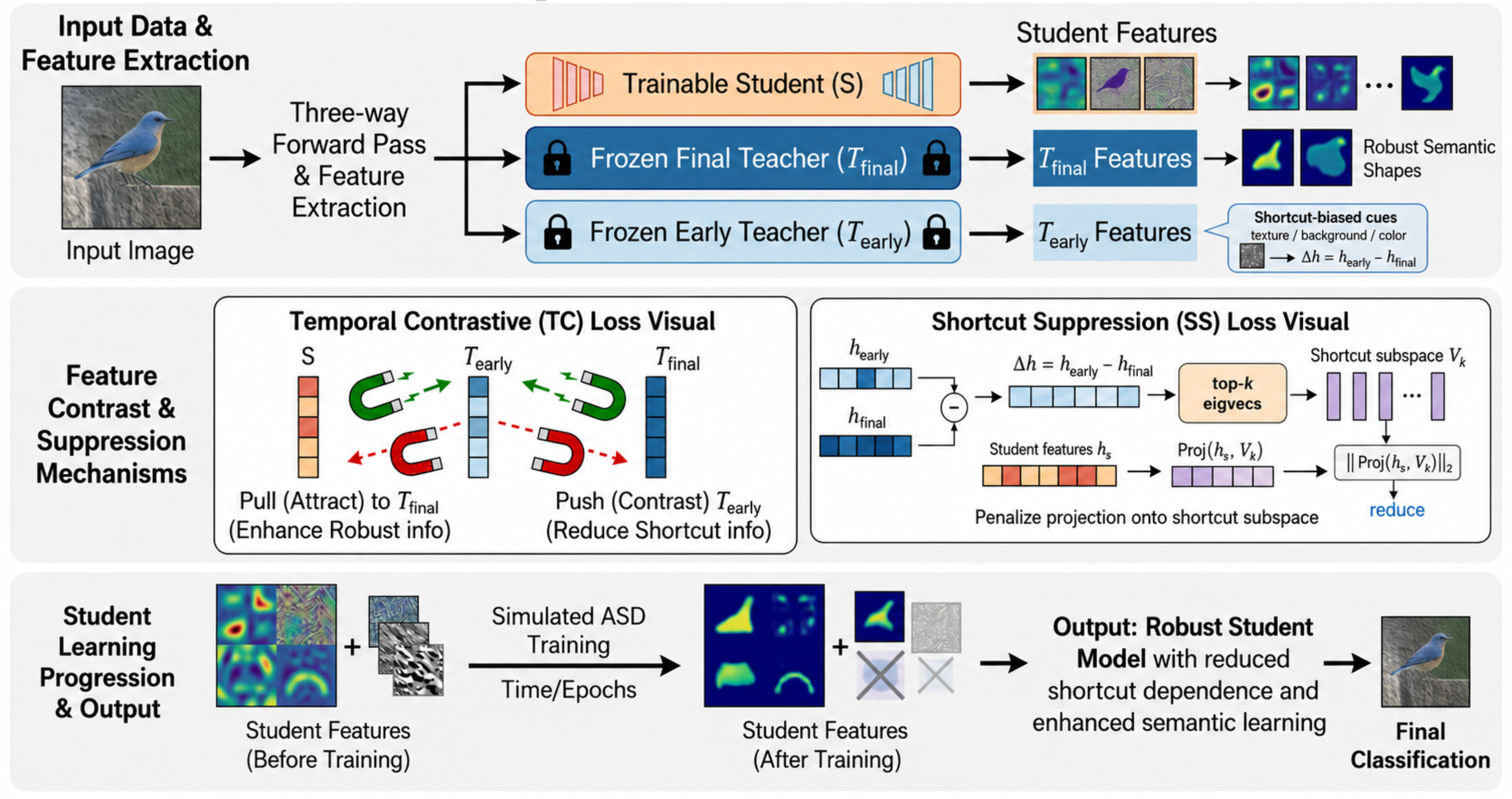}
  \caption{Anti-Shortcut Distillation (ASD) trains a student
  using two frozen teacher checkpoints: a converged teacher $\Tfinal$
  providing robust semantic targets and an early teacher $\Tearly$ exposing
  shortcut-biased cues. $\Ltc$ attracts the student toward $\Tfinal$ while
  repelling it from $\Tearly$; $\Lss$ penalizes projection onto the principal
  shortcut subspace defined by $\Dh = \hearly - \hfinal$.}
  \label{fig:pipeline}
  \vspace{-3mm}
\end{figure*}

\vspace{-4mm}
\paragraph{Contributions.}
First, we propose a signed KD framework that treats $\Tfinal$ as a positive
semantic anchor and $\Tearly$ as a temporal negative reference, combining
final-teacher imitation with explicit shortcut avoidance through $\Ltc$.
Second, we identify $\Dh = \hearly - \hfinal$ as a label-free shortcut
indicator and suppress student projection onto the leading eigenvectors of
$\E[\Dh\Dh^{\top}]$ through $\Lss$. Third, we provide a compact analysis
showing that, under a two-subspace decomposition of teacher features, the
top eigenspace of $\E[\Dh\Dh^{\top}]$ recovers the shortcut subspace, with a
Davis--Kahan~\cite{davis1970rotation,yu2015useful} perturbation bound on the
empirical (\cref{thm:shortcut_id}). Finally, across 13
teacher--student pairs on CIFAR-100, ImageNet-100, and TinyImageNet, ASD
attains the best clean accuracy on 10 pairs and surpasses KD on 12;
mechanistic diagnostics show ASD students are anti-aligned with $\Dh$
($\cos(\hs,\Dh) = -0.36$ vs.\ $0.00$ for KD) while increasing
robust-subspace projection ($0.45$ vs.\ $0.12$).

\vspace{-4mm}

\section{Related Work}
\label{sec:related}

Response-based KD~\cite{hinton2015distilling,gou2021knowledge} aligns
softened logits; feature/attention
methods~\cite{romero2014fitnets,zagoruyko2016paying} align hidden
activations; relational
methods~\cite{park2019relational,passalis2018learning} preserve sample
geometry; and recent objectives~\cite{tian2019contrastive,zhao2022decoupled,
chen2021distilling} broaden transfer through contrastive, decoupled, or
hierarchical alignment. All share an attractive premise. ASD instead uses
the teacher to define what should be suppressed, not only imitated.
TAKD~\cite{mirzadeh2020improved}, CkptKD~\cite{jin2022efficient}, snapshot
distillation~\cite{huang2017snapshot}, Born-Again
Networks~\cite{furlanello2018born}, and BYOT~\cite{zhang2019your} all
recognise that intermediate teacher states contain useful supervision and
treat them as positive targets. ASD inverts the sign: $\Tearly$ is a
temporal negative that exposes shortcut-prone directions the student should
be pushed away from.
Lukasik~\etal~\cite{lukasik2021teacher} show distillation amplifies teacher
errors on subgroups; Lee~\&~Lee~\cite{lee2023debiased} debias by
transplanting the teacher's last layer. Both diagnose when transfer is
harmful at the label/head level; ASD operationalizes the harmful
directions, extracting them label-free from the early-to-final displacement.
Networks fit simple, dominant patterns before more robust
structure~\cite{arpit2017closer,shah2020pitfalls}; in image classification
this manifests as texture
bias~\cite{geirhos2019imagenet,geirhos2020shortcut}, weakening robustness
under shift~\cite{sagawa2020distributionally,hendrycks2019benchmarking}.
ASD targets this failure mode at distillation time using the teacher's own
early-to-final feature change as a label-free shortcut indicator.
SimCLR~\cite{chen2020simple}, MoCo~\cite{he2020momentum}, and memory-bank
methods~\cite{wu2018unsupervised} learn invariances by separating positives
from negatives; supervised contrastive
learning~\cite{khosla2020supervised} adds label structure;
CRD~\cite{tian2019contrastive} brings contrast into KD with sample-level
negatives. ASD moves the contrastive axis to teacher time: positives
and negatives are distinct optimization states of the same teacher.

\vspace{-4mm}
\section{Method}
\label{sec:method}
ASD is a training-time distillation framework built around a three-way
forward pass: the trainable student, the frozen converged teacher, and a
frozen early teacher checkpoint. The final teacher provides the semantic
target to imitate, whereas the early teacher exposes shortcut-biased
directions that should not be transferred to the student. This section
describes how ASD converts this teacher trajectory into two coupled
mechanisms: temporal contrastive alignment (\cref{sec:ltc}) and
eigenvector-based shortcut suppression (\cref{sec:lss}).

\vspace{-4mm}
\subsection{Problem Setup}
\label{sec:setup}

Let $f_T : \mathcal{X} \to \R^{C}$ be a teacher classifier over $C$ classes,
and let $\phi_T(x) \in \R^{d_T}$ denote its penultimate feature for input
$x$. From the teacher training trajectory we retain two frozen snapshots:
the final teacher $\Tfinal$ and an early checkpoint $\Tearly$. The student
$f_S$ has the same output dimension $C$, but may use a different
architecture with pre-logit feature $\phi_S(x) \in \R^{d_S}$. We define
\begin{equation}
  \hfinal(x) = \phi_{\Tfinal}(x), \qquad
  \hearly(x) = \phi_{\Tearly}(x).
  \label{eq:teacher_features}
\end{equation}
When $d_S \neq d_T$, a trainable projector
$g : \R^{d_S} \to \R^{d_T}$ maps student features into the teacher feature
space; when $d_S = d_T$, $g$ is the identity. The student feature used by
ASD is therefore
\begin{equation}
  \hs(x) = g(\phi_S(x)).
  \label{eq:student_feature}
\end{equation}
Both teacher networks remain frozen throughout student training; only the
student and, when needed, the projector are optimized. The projector is
discarded at inference, so the deployed model is the student alone.

\vspace{-4mm}
\subsection{Temporal Shortcut Signal}
\label{sec:shortcut_id}

Standard KD observes only the final teacher state. ASD instead uses the
teacher's temporal change as supervision. Early teacher features tend to
retain low-level, shortcut-prone information that the converged teacher
later reduces or reorganizes into more stable semantic representations. For
each input, ASD measures this change through the early-to-final
displacement
\begin{equation}
  \Dh(x) = \hearly(x) - \hfinal(x).
  \label{eq:delta_h}
\end{equation}
This vector is not treated as noise. It is the central anti-shortcut signal:
components emphasized by $\Tearly$ but attenuated by $\Tfinal$ become
candidate directions to suppress during student learning.
For a mini-batch $\mathcal{B}$, ASD forms the uncentered second-moment matrix
\begin{equation}
  \Cschat \;=\; \frac{1}{|\mathcal{B}|} \sum_{x_i \in \mathcal{B}}
  \Dh(x_i)\,\Dh(x_i)^{\top}
  \;\in\; \R^{d_T \times d_T}.
  \label{eq:shortcut_cov}
\end{equation}
The matrix is deliberately uncentered: the mean displacement may
itself encode a dominant shortcut direction, and centering would remove
that signal. Let $\Uk = [u_1, \ldots, u_K]$ denote the top-$K$ orthonormal
eigenvectors of $\Cschat$. These eigenvectors span the
shortcut subspace used by ASD in the current training step. Although $\Cschat$ is rank-deficient at typical batch sizes
($B \ll d_T$), $\Lss$ uses only the $K{=}4 \ll B$ dimensional subspace
$\mathrm{span}(\Uk)$ and is invariant to rotations within it; accuracy and
robustness are insensitive to $B \in \{32,\dots,256\}$, while subspace
agreement improves with $B$ at the rate predicted by
\cref{sec:theory_subspace} (Supplementary Sec.~A.4.5). The same
temporal displacement now serves two roles: it defines a negative teacher
state for contrastive learning (\cref{sec:ltc}) and a geometric subspace
for shortcut suppression (\cref{sec:lss}).

\vspace{-4mm}
\subsection{Temporal Contrastive Loss}
\label{sec:ltc}

The Temporal Contrastive Loss makes the student's representation closer to
the final teacher than to the early teacher for the same input. For a
mini-batch $\{x_i\}_{i=1}^{B}$, let $\zs^{i}$, $\zf^{i}$, and $\ze^{i}$ be
the $\ell_{2}$-normalized versions of $\hs(x_i)$, $\hfinal(x_i)$, and
$\hearly(x_i)$, respectively. ASD treats $\zf^{i}$ as the positive target
for $\zs^{i}$ and $\ze^{i}$ as the temporal negative. Other final-teacher
features inside the mini-batch and a FIFO memory queue
$\mathcal{M} = \{q_m\}_{m=1}^{M}$ of normalized $\Tfinal$ features provide
additional negatives. The loss for sample $x_i$ is
\begin{equation}
  \Ltc^{(i)} \;=\; -\log
  \frac{ \exp\!\left( \zs^{i} \cdot \zf^{i} / \tau_c \right) }{ D_i },
  \label{eq:ltc}
\end{equation}
where the denominator $D_i$ assembles one positive and three types of
negatives,
\begin{align}
  D_i \;=\;&
  \underbrace{\exp\!\left( \zs^{i} \cdot \zf^{i} / \tau_c \right)}_{\text{positive}}
  +
  \underbrace{\exp\!\left( \zs^{i} \cdot \ze^{i} / \tau_c \right)}_{\text{temporal negative}}
  \nonumber \\
  & +
  \underbrace{\sum_{j \neq i} \exp\!\left( \zs^{i} \cdot \zf^{j} / \tau_c \right)}_{\text{in-batch final-teacher}}
  +
  \underbrace{\sum_{m=1}^{M} \exp\!\left( \zs^{i} \cdot q_m / \tau_c \right)}_{\text{memory-bank final-teacher}}.
  \label{eq:ltc_denominator}
\end{align}
The positive term pulls the student toward the converged teacher; the
same-sample temporal negative makes the early shortcut state explicitly
competitive; the in-batch and memory-bank terms prevent the contrastive
problem from degenerating into a binary comparison and provide a large
negative set required for an informative InfoNCE
bound~\cite{oord2018representation,poole2019variational}. Replacing the
temporal negative with random in-batch negatives recovers a CRD-style
objective; the resulting ablation
(\cref{sec:ablation_negative_source}) isolates the contribution of
$\Tearly$.

\vspace{-4mm}
\subsection{Shortcut Suppression Loss}
\label{sec:lss}

Temporal contrast reduces similarity to $\Tearly$, but it does not directly
constrain the subspace in which student features may lie. ASD therefore
adds a geometric suppression term. Given the shortcut eigenspace $\Uk$, the
normalized student feature is projected onto this subspace and penalized
when the projection magnitude exceeds a margin $\varepsilon > 0$:
\begin{equation}
  \Lss \;=\; \frac{1}{B} \sum_{i=1}^{B}
  \max\!\left( 0, \;
  \norm{ \hSbar(x_i)^{\top} \Uk }_{2} - \varepsilon \right),
  \qquad \text{where} \quad
  \hSbar(x_i) = \frac{\hs(x_i)}{\norm{\hs(x_i)}_2}.
  \label{eq:lss}
\end{equation}
Here $\norm{\hSbar^{\top}\Uk}_2 \in [0,1]$ is the (unsigned) magnitude of
the student feature's projection onto $\mathrm{span}(\Uk)$; the hinge is
inactive below $\varepsilon$ and grows linearly above it. We use
$\varepsilon = 0.1$ throughout. We stress that this projection margin is
distinct from the signed-cosine threshold $\varepsilon_{\mathrm{diag}}=-0.1$
used only as a diagnostic in \cref{sec:mechanistic_validation}; the former
constrains a subspace projection during training, the latter measures signed
alignment with a single direction at evaluation. While $\Ltc$ supplies a
relative push away from the early teacher at the sample level, $\Lss$
imposes an explicit eigenspace constraint at the batch level. Together they
prevent the student from matching the final teacher while silently
recovering the same shortcuts from data: $\Ltc$ alone cannot forbid the
student from re-discovering shortcut directions, and $\Lss$ alone carries no
per-sample temporal signal.

\vspace{-4mm}
\subsection{Projector for Heterogeneous Architectures}
\label{sec:projector}

For heterogeneous teacher--student pairs ($d_S \neq d_T$), direct feature
comparison is undefined. We use the trainable projector
\begin{equation}
  g \;=\; \mathrm{Linear}(d_S, d_T) \to \mathrm{BN}(d_T) \to \mathrm{ReLU}
  \to \mathrm{Linear}(d_T, d_T).
  \label{eq:projector}
\end{equation}
The projector is optimized jointly with the student and used only to
compute feature-level ASD losses; it is removed at inference, keeping the
deployed model architecture-unchanged.

\vspace{-4mm}
\subsection{Training Objective}
\label{sec:total_loss}

ASD combines task supervision, final-teacher logit distillation, temporal
contrast, and shortcut-subspace suppression:
\begin{equation}
  \Lasd \;=\; \Lce \;+\; \alpha_{\mathrm{kd}}\Lkd
  \;+\; \beta(t)\,\alpha_{\mathrm{tc}}\Ltc
  \;+\; \gamma(t)\,\alpha_{\mathrm{ss}}\Lss.
  \label{eq:total}
\end{equation}
Here $\Lce$ is cross-entropy and $\Lkd$ is temperature-scaled logit
distillation~\cite{hinton2015distilling} from $\Tfinal$. Because the student
representation is unstable early in training, ASD warms up the two
auxiliary terms linearly:
\begin{equation}
  \beta(t) \;=\; \gamma(t) \;=\; \min\!\left( \frac{t}{T_{\mathrm{warmup}}}, \; 1 \right).
  \label{eq:warmup}
\end{equation}
The resulting objective is a push--pull distillation process: the student is
pulled toward the final teacher's semantic representation while being
pushed away from shortcut directions exposed by the early teacher.

\vspace{-4mm}
\section{Theoretical Analysis}
\label{sec:theory}

We give a compact analysis of ASD at the population level and defer
extended proofs to the supplementary material. The aim is not to make hard
generalization guarantees, but to give a precise meaning to the two losses
and to the role of $\Dh$ as a shortcut indicator. We state one
information-theoretic interpretation of $\Ltc$
(\cref{prop:infonce_alignment}), one geometric reading of $\Lss$
(\cref{prop:projection_margin}), and one structural theorem
(\cref{thm:shortcut_id}) that justifies $\Dh$ as a shortcut indicator.

\vspace{-5mm}
\subsection{Information-Theoretic Interpretation of \texorpdfstring{$\Ltc$}{L\_TC}}
\label{sec:theory_tc}

The InfoNCE objective lower-bounds the mutual information between the two
endpoints of the positive pair, with a $\log N$ saturation term
depending on the number of marginal negatives~\cite{oord2018representation,
poole2019variational}. We give the version of this bound that applies to
$\Ltc$, with the temporal negative correctly handled.

\begin{proposition}[InfoNCE alignment with $\Tfinal$]
\label{prop:infonce_alignment}
Let $\widetilde{\Ltc}$ denote the variant of $\Ltc$ in which the same-sample
temporal-negative term $\exp(\zs^{i}\cdot\ze^{i}/\tau_c)$ is removed from
the denominator of \cref{eq:ltc_denominator}, leaving only the positive
and the $|\mathcal{N}_F| = (B-1) + M$ in-batch and memory-bank
\emph{final-teacher} negatives, drawn i.i.d.\ from the marginal
$p_{\hfinal}$. Then
\begin{equation}
  I(\hs;\hfinal) \;\geq\; \log\bigl(|\mathcal{N}_F| + 1\bigr) - \E[\widetilde{\Ltc}].
  \label{eq:infonce_bound}
\end{equation}
Furthermore, $\Ltc \geq \widetilde{\Ltc}$ pointwise, so
$\E[\widetilde{\Ltc}] \leq \E[\Ltc]$ and the same bound is preserved when
$\Ltc$ replaces $\widetilde{\Ltc}$.
\end{proposition}

\begin{proof}[Proof sketch]
\Cref{eq:infonce_bound} is the standard InfoNCE
bound~\cite{oord2018representation,poole2019variational} applied with
positive pair $(\hs, \hfinal)$ and negatives drawn i.i.d.\ from
$p_{\hfinal}$. Adding an extra positive term to the denominator of a
softmax can only decrease the log-probability of the positive class, so
$\Ltc \geq \widetilde{\Ltc}$. The full proof and a discussion of bound
tightness are in the supplementary.
\end{proof}

\noindent The proposition gives the temporal negative an operational
reading: it does not change the MI bound on $\Tfinal$ (the bound is
controlled by the marginal $\Tfinal$ negatives), but it adds an explicit
penalty on similarity to $\Tearly$ \emph{at no cost} to the bound. This is
precisely how ASD differs from CRD~\cite{tian2019contrastive}: the
same-sample temporal negative changes the geometry of the contrastive
decision, not merely the count of negatives.

\vspace{-4mm}
\subsection{Geometric Interpretation of \texorpdfstring{$\Lss$}{L\_SS}}
\label{sec:theory_ss}

\begin{proposition}[Angular margin from shortcut suppression]
\label{prop:projection_margin}
Let $\theta(x) \in [0, \tfrac{\pi}{2}]$ be the angle between $\hSbar(x)$ and
its orthogonal projection onto $\mathrm{span}(\Uk)$, i.e.\ the angle to the
nearest point of the subspace. Since $\hSbar(x)$ is unit-norm and $\Uk$ has
orthonormal columns, $\norm{\hSbar(x)^{\top}\Uk}_2 = \cos\theta(x)$
exactly. Consequently, if $\Lss = 0$ on a sample $x$ then
$\cos\theta(x) \le \varepsilon$, equivalently
$\theta(x) \ge \arccos(\varepsilon)$. With $\varepsilon = 0.1$ this enforces
$\theta(x) \ge 84.3^{\circ}$.
\end{proposition}

The proposition gives $\Lss$ a direct geometric reading: when it vanishes,
the student feature is separated from \emph{every} direction in the
shortcut subspace by at least $\arccos(\varepsilon)$, a strictly
stronger constraint than the relative push provided by $\Ltc$
(\cref{sec:ablation_components}).

\vspace{-4mm}
\subsection{When Does \texorpdfstring{$\Dh$}{Delta-h} Identify the Shortcut Subspace?}
\label{sec:theory_subspace}

We formalize when the top eigenspace of $\E[\Dh\Dh^{\top}]$ identifies a
meaningful shortcut subspace through a two-subspace decomposition of
teacher features.

\begin{assumption}[Two-subspace decomposition]
\label{assum:two_subspace}
Teacher features at any state $t \in \{\mathrm{early},\mathrm{final}\}$
admit the decomposition
\begin{equation}
  h_t(x) \;=\; r_t(x) + s_t(x),
  \qquad r_t(x) \in \Vrob, \;\; s_t(x) \in \Vsc,
  \label{eq:two_subspace}
\end{equation}
where $\Vrob$ and $\Vsc$ are orthogonal complementary subspaces of $\R^{d_T}$.
\end{assumption}

\begin{theorem}[Shortcut subspace identification]
\label{thm:shortcut_id}
Under assumption \ref{assum:two_subspace}, let
$G_r = \E[(r_{\mathrm{early}} - r_{\mathrm{final}}) \\
(r_{\mathrm{early}} - r_{\mathrm{final}})^{\top}]$ and
$G_s = \E[(s_{\mathrm{early}} - s_{\mathrm{final}})
(s_{\mathrm{early}} - s_{\mathrm{final}})^{\top}]$. If the shortcut
component dominates,
\begin{equation}
  \lambda_{K}(G_s) \;>\; \norm{G_r}_2,
  \label{eq:shortcut_energy}
\end{equation}
then the top-$K$ eigenspace of $\E[\Dh\Dh^{\top}]$ coincides with the
top-$K$ eigenspace of $G_s$, which lies entirely in $\Vsc$. Consequently,
minimizing $\Lss$ reduces student projection onto $\Vsc$.
\end{theorem}

\begin{proof}[Proof sketch]
Under $\Vrob \perp \Vsc$, the cross terms in $\E[\Dh\Dh^{\top}]$ vanish, so
$\E[\Dh\Dh^{\top}] = G_r \oplus G_s$ is block-diagonal in the basis
$\Vrob \cup \Vsc$. Eigenvectors of a block-diagonal symmetric matrix are
contained in the individual blocks; \cref{eq:shortcut_energy} forces the
top-$K$ eigenvalues to come from $G_s$, so the top-$K$ eigenspace lies in
$\Vsc$. For the empirical $\Cschat$, Davis--Kahan
sin-$\Theta$~\cite{davis1970rotation,yu2015useful} bounds the deviation
from the population subspace by
$\norm{\Csc - \Cschat}_2 / (\lambda_K(G_s) - \norm{G_r}_2)$. The full proof
is in the supplementary.
\end{proof}
\vspace{-3mm}
\noindent\textit{Interpretation.} \Cref{eq:shortcut_energy} states that
between early and final teacher states the dominant change in feature
variance occurs in the shortcut subspace, not the robust subspace. This is
the empirical regime that motivates ASD, and we verify it directly: the
principal angle between the shortcut subspace and the
$\Tfinal$-PCA subspace is $65.3^{\circ}$ against $38.6^{\circ}$ to the
$\Tearly$-PCA subspace (\cref{sec:mechanistic_validation}).
Across the teacher trajectory (Supplementary Fig.~A.6), shortcut energy
in $\mathrm{span}(\Uk)$ peaks at $\Tearly$ ($0.25$) and drops by ${\sim}49\%$
at convergence ($0.13$), while robust $\Tfinal$-PCA energy remains stable
(${\sim}0.11$--$0.13$), supporting the premise of $\Dh$.

\subsection{Robustness under Shortcut-Aligned Shifts}
\label{sec:theory_robustness}

If $\Lss$ is small, the student is confined to a region whose projection
onto $\mathrm{span}(\Uk)$ is at most $\varepsilon$. For a shift that
perturbs features within $\mathrm{span}(\Uk)$ and a locally-linear
classifier head, the first-order change in classifier output is bounded by
this projection. ASD therefore exhibits reduced first-order sensitivity
to shortcut-aligned shifts compared with a KD student of larger shortcut
projection. The linearity and shift-support conditions are restrictive; we
treat this as a design hypothesis and evaluate it empirically in
\cref{sec:robustness,sec:mechanistic_validation}.

\vspace{-4mm}
\section{Experiments}
\label{sec:experiments}

\vspace{-2mm}
\subsection{Experimental Setup}
\label{sec:setup_exp}


\vspace{-2mm}
\paragraph{Datasets and shortcut benchmarks.}
CIFAR-100~\cite{krizhevsky2009learning} provides the main
controlled distillation setting (50K/10K train/test, 100 classes).
CIFAR-100-C~\cite{hendrycks2019benchmarking} contains 15 corruption types
across 5 severity levels and is used only for robustness evaluation.
ImageNet-100 (a 100-class subset of ImageNet~\cite{deng2009imagenet})
evaluates ASD at larger image resolution and data scale. TinyImageNet
stresses the method with 200 classes, 100K/10K images, and $64\times64$
inputs. Corruptions alone are not shortcut benchmarks, so we
additionally evaluate the trained ImageNet-100 students
(ResNet-50$\to$MobileNet-V2), without retraining, on
\emph{Stylized-ImageNet-100} (AdaIN~\cite{geirhos2019imagenet}, full
100-class val set), \emph{ImageNet-R}~\cite{hendrycks2021many} (19
overlapping classes), and \emph{ImageNet-A}~\cite{hendrycks2021nae}
(15-class overlap, 680 images); protocol and per-seed results in
Supplementary Sec.~A.3.9.
\vspace{-4mm}
\paragraph{Baselines.}
We compare with CE-only, KD~\cite{hinton2015distilling},
FitNets~\cite{romero2014fitnets}, AT~\cite{zagoruyko2016paying},
DKD~\cite{zhao2022decoupled}, CRD~\cite{tian2019contrastive}, and
CkptKD~\cite{jin2022efficient}, covering logit transfer, feature
regression, attention alignment, decoupled logits, contrastive
representation distillation, and checkpoint reuse. All methods share the
optimizer, schedule, augmentation, and training budget.
\vspace{-4mm}
\paragraph{Training details.}
CIFAR-100: 240 epochs, SGD (momentum 0.9, weight decay $5\times10^{-4}$),
initial LR 0.05, cosine schedule, batch 64. ImageNet-100: 100 epochs, LR
0.1, batch 256. TinyImageNet: 200 epochs. $\Tearly$ is the 15\%-of-training
checkpoint on every dataset. ASD hyperparameters:
$\alpha_{\mathrm{kd}}=1.0$, $\alpha_{\mathrm{tc}}=0.8$,
$\alpha_{\mathrm{ss}}=1.0$, $\tau_{\mathrm{kd}}=4.0$, $\tau_c=0.07$,
$\varepsilon=0.1$, $K=4$, $M=4096$, $T_{\mathrm{warmup}}=20$ epochs.
Robustness is mean Corruption Error (mCE) on
CIFAR-100-C~\cite{hendrycks2019benchmarking} across all 15 corruption
types at severity 5. Results are mean$\pm$std over 3 seeds unless noted.

\vspace{-4mm}
\subsection{Same-Family CIFAR-100 Results}
\label{sec:same_family}
\begin{table}[ht]
\footnotesize
\setlength{\tabcolsep}{3.5pt}
\begin{center}
\resizebox{\linewidth}{!}{
\begin{tabular}{l|c|c|c|c|c}
\toprule
\thead{Teacher \\ Student} & \thead{ResNet-34 \\ ResNet-18} & \thead{WRN-40-2 \\ WRN-16-2} & \thead{ResNet-32$\times$4 \\ ResNet-8$\times$4} & \thead{WRN-40-2 \\ WRN-40-1} & \thead{VGG-13 \\ VGG-8} \\
\midrule
Teacher & 79.91 & 76.49 & 79.72 & 76.49 & 76.15 \\
Student & 78.79 & 73.33 & 71.90 & 71.49 & 71.10 \\
\midrule
KD~\cite{hinton2015distilling} & 80.22 $\pm$ 0.19 ($\mathcolor{red}{\downarrow}$) & 74.74 $\pm$ 0.17 ($\mathcolor{red}{\downarrow}$) & 73.32 $\pm$ 0.11 ($\mathcolor{red}{\downarrow}$) & 73.38 $\pm$ 0.47 ($\mathcolor{red}{\downarrow}$) & 73.68 $\pm$ 0.35 ($\mathcolor{red}{\downarrow}$) \\
FitNets~\cite{romero2014fitnets} & 80.32 $\pm$ 0.17 ($\mathcolor{ForestGreen}{\uparrow}$) & 74.94 $\pm$ 0.24 ($\mathcolor{ForestGreen}{\uparrow}$) & 73.55 $\pm$ 0.04 ($\mathcolor{ForestGreen}{\uparrow}$) & 73.54 $\pm$ 0.18 ($\mathcolor{ForestGreen}{\uparrow}$) & 73.91 $\pm$ 0.11 ($\mathcolor{ForestGreen}{\uparrow}$) \\
AT~\cite{zagoruyko2016paying} & 80.30 $\pm$ 0.12 ($\mathcolor{ForestGreen}{\uparrow}$) & 74.99 $\pm$ 0.15 ($\mathcolor{ForestGreen}{\uparrow}$) & \best{73.78 $\pm$ 0.15} ($\mathcolor{ForestGreen}{\uparrow}$) & 73.79 $\pm$ 0.16 ($\mathcolor{ForestGreen}{\uparrow}$) & 71.61 $\pm$ 0.12 ($\mathcolor{red}{\downarrow}$) \\
DKD~\cite{zhao2022decoupled} & 80.04 $\pm$ 0.36 ($\mathcolor{red}{\downarrow}$) & 75.08 $\pm$ 0.14 ($\mathcolor{ForestGreen}{\uparrow}$) & 73.69 $\pm$ 0.32 ($\mathcolor{ForestGreen}{\uparrow}$) & 73.63 $\pm$ 0.07 ($\mathcolor{ForestGreen}{\uparrow}$) & 73.99 $\pm$ 0.05 ($\mathcolor{ForestGreen}{\uparrow}$) \\
CRD~\cite{tian2019contrastive} & 80.71 $\pm$ 0.34 ($\mathcolor{ForestGreen}{\uparrow}$) & 74.76 $\pm$ 0.12 ($\mathcolor{ForestGreen}{\uparrow}$) & 73.15 $\pm$ 0.12 ($\mathcolor{red}{\downarrow}$) & 73.63 $\pm$ 0.33 ($\mathcolor{ForestGreen}{\uparrow}$) & 74.04 $\pm$ 0.18 ($\mathcolor{ForestGreen}{\uparrow}$) \\
CkptKD~\cite{jin2022efficient} & 73.94 $\pm$ 0.13 ($\mathcolor{red}{\downarrow}$) & 71.31 $\pm$ 0.06 ($\mathcolor{red}{\downarrow}$) & 72.17 $\pm$ 0.15 ($\mathcolor{red}{\downarrow}$) & 70.65 $\pm$ 0.15 ($\mathcolor{red}{\downarrow}$) & 70.48 $\pm$ 0.12 ($\mathcolor{red}{\downarrow}$) \\
\textbf{ASD (ours)} & \best{80.50 $\pm$ 0.07} ($\mathcolor{ForestGreen}{\uparrow}$) & \best{75.87 $\pm$ 0.29} ($\mathcolor{ForestGreen}{\uparrow}$) & 73.13 $\pm$ 0.12 ($\mathcolor{red}{\downarrow}$) & \best{74.16 $\pm$ 0.21} ($\mathcolor{ForestGreen}{\uparrow}$) & \best{74.26 $\pm$ 0.13} ($\mathcolor{ForestGreen}{\uparrow}$) \\
\bottomrule
\end{tabular}
}
\vspace{-4pt}
\caption{\small{Top-1 accuracy (\%) on CIFAR-100 for same-family teacher--student pairs. P1--P5 denote the five teacher--student pairs shown in the header. $\mathcolor{ForestGreen}{\uparrow}$\,/\,$\mathcolor{red}{\downarrow}$ indicates better/worse performance than KD. Grey bold marks the best result per column. Results are mean\,$\pm$\,std over 3 seeds.}}
\label{tbl:cifar100_same}
\end{center}
\vspace{-8pt}
\end{table}

Same-family pairs isolate capacity compression under shared architectural
bias (\cref{tbl:cifar100_same}). ASD improves over KD on 4 of 5 pairs (P1,
P2, P4, P5) and obtains the best accuracy on P2 (75.87\%), P4 ($74.16\%$) and P5
($74.26\%$). On P1, CRD edges ASD by $0.21$ pp on clean accuracy, but with
nearly $5\times$ larger seed variance ($\pm0.34$ vs.\ $\pm0.07$); we will
see in \cref{sec:robustness} that ASD's robustness on P1 is comparable to
CRD's while preserving lower seed-level variability.

The single same-family failure is P3
(ResNet-32$\times$4$\to$ResNet-8$\times$4), where ASD reaches $73.13\%$,
slightly below KD ($73.32\%$) and AT ($73.78\%$). This is the
most aggressively compressed student in our sweep. We discuss the capacity
boundary that this case represents in \cref{sec:p3}.

\vspace{-14pt}
\subsection{Cross-Architecture CIFAR-100 Results}
\label{sec:cross_arch}
\vspace{-2mm}
\begin{table}[ht]
\footnotesize
\setlength{\tabcolsep}{3.5pt}
\begin{center}
\resizebox{\linewidth}{!}{
\begin{tabular}{l|c|c|c|c|c}
\toprule
\thead{Teacher \\ Student} & \thead{WRN-40-2 \\ ShuffleNet-V2} & \thead{WRN-40-2 \\ ShuffleNet-V1} & \thead{ResNet-32$\times$4 \\ ShuffleNet-V1} & \thead{ResNet-32$\times$4 \\ ShuffleNet-V2} & \thead{VGG-13 \\ MobileNet-V2} \\
\midrule
Teacher & 76.49 & 76.49 & 79.72 & 79.72 & 76.15 \\
Student & 63.26 & 71.46 & 71.46 & 63.26 & 64.12 \\
\midrule
CE only & 63.26 $\pm$ 0.27 & 71.46 $\pm$ 0.38 & 71.46 $\pm$ 0.38 & 63.26 $\pm$ 0.27 & 64.12 $\pm$ 0.31 \\
KD~\cite{hinton2015distilling} & 66.72 $\pm$ 0.32 ($\mathcolor{red}{\downarrow}$) & 73.79 $\pm$ 0.26 ($\mathcolor{red}{\downarrow}$) & 73.09 $\pm$ 0.38 ($\mathcolor{red}{\downarrow}$) & 66.17 $\pm$ 0.18 ($\mathcolor{red}{\downarrow}$) & 65.86 $\pm$ 0.28 ($\mathcolor{red}{\downarrow}$) \\
FitNets~\cite{romero2014fitnets} & 67.07 $\pm$ 0.15 ($\mathcolor{ForestGreen}{\uparrow}$) & 74.27 $\pm$ 0.04 ($\mathcolor{ForestGreen}{\uparrow}$) & 73.20 $\pm$ 0.50 ($\mathcolor{ForestGreen}{\uparrow}$) & 66.07 $\pm$ 0.36 ($\mathcolor{red}{\downarrow}$) & 66.18 $\pm$ 0.22 ($\mathcolor{ForestGreen}{\uparrow}$) \\
AT~\cite{zagoruyko2016paying} & 67.66 $\pm$ 0.19 ($\mathcolor{ForestGreen}{\uparrow}$) & 73.97 $\pm$ 0.21 ($\mathcolor{ForestGreen}{\uparrow}$) & 73.25 $\pm$ 0.28 ($\mathcolor{ForestGreen}{\uparrow}$) & 66.28 $\pm$ 0.24 ($\mathcolor{ForestGreen}{\uparrow}$) & 66.38 $\pm$ 0.19 ($\mathcolor{ForestGreen}{\uparrow}$) \\
DKD~\cite{zhao2022decoupled} & 66.84 $\pm$ 0.31 ($\mathcolor{ForestGreen}{\uparrow}$) & 73.63 $\pm$ 0.39 ($\mathcolor{red}{\downarrow}$) & 72.33 $\pm$ 0.33 ($\mathcolor{red}{\downarrow}$) & 66.42 $\pm$ 0.28 ($\mathcolor{ForestGreen}{\uparrow}$) & 66.02 $\pm$ 0.24 ($\mathcolor{ForestGreen}{\uparrow}$) \\
CRD~\cite{tian2019contrastive} & 67.81 $\pm$ 0.42 ($\mathcolor{ForestGreen}{\uparrow}$) & 74.34 $\pm$ 0.47 ($\mathcolor{ForestGreen}{\uparrow}$) & 73.60 $\pm$ 0.16 ($\mathcolor{ForestGreen}{\uparrow}$) & 66.98 $\pm$ 0.36 ($\mathcolor{ForestGreen}{\uparrow}$) & 66.64 $\pm$ 0.31 ($\mathcolor{ForestGreen}{\uparrow}$) \\
CkptKD~\cite{jin2022efficient} & 67.27 $\pm$ 0.03 ($\mathcolor{ForestGreen}{\uparrow}$) & 72.66 $\pm$ 0.20 ($\mathcolor{red}{\downarrow}$) & 73.39 $\pm$ 0.01 ($\mathcolor{ForestGreen}{\uparrow}$) & 66.52 $\pm$ 0.08 ($\mathcolor{ForestGreen}{\uparrow}$) & 66.08 $\pm$ 0.14 ($\mathcolor{ForestGreen}{\uparrow}$) \\
\textbf{ASD (ours)} & \best{68.05 $\pm$ 0.14} ($\mathcolor{ForestGreen}{\uparrow}$) & \best{74.94 $\pm$ 0.09} ($\mathcolor{ForestGreen}{\uparrow}$) & \best{73.98 $\pm$ 0.27} ($\mathcolor{ForestGreen}{\uparrow}$) & \best{67.18 $\pm$ 0.19} ($\mathcolor{ForestGreen}{\uparrow}$) & \best{67.02 $\pm$ 0.18} ($\mathcolor{ForestGreen}{\uparrow}$) \\
\bottomrule
\end{tabular}
}
\vspace{-4pt}

\caption{\small{Top-1 accuracy (\%) on CIFAR-100 for cross-architecture teacher--student pairs. P6--P10 denote the five pairs shown in the header. $\mathcolor{ForestGreen}{\uparrow}$\,/\,$\mathcolor{red}{\downarrow}$ indicates better/worse performance than KD. Grey bold marks the best result per column. Results are mean\,$\pm$\,std over 3 seeds.}}

\label{tbl:cifar100_cross}
\end{center}
\vspace{-8pt}
\end{table}

ASD attains the best accuracy on \emph{all} five heterogeneous pairs
(\cref{tbl:cifar100_cross}), with the strongest absolute gain over KD on
WRN-40-2$\to$ShuffleNet-V2 ($+1.33$ pp). This is the regime where
push--pull supervision matters most: when teacher and student have
different representational bases, purely attractive transfer leaves the
student free to find architecture-specific shortcuts that reproduce teacher
predictions without recovering the teacher's robust geometry. ASD removes
this ambiguity by explicitly designating directions the student should not
occupy. ASD also exceeds CRD on every cross-architecture pair, including
the difficult WRN-40-2$\to$ShuffleNet-V2 case where it improves over CRD
by $0.24$ pp; sample-level contrastive negatives alone are not sufficient
in this regime.

\vspace{-4mm}
\subsection{Beyond CIFAR-100}
\label{sec:beyond_cifar}


\begin{table}[ht]
\footnotesize
\setlength{\tabcolsep}{5pt}
\begin{center}
\begin{tabular}{l|c|c|c}
\toprule
\thead{Teacher $\to$ Student \\ Dataset} &
  \thead{ResNet-50 $\to$ MobileNet-V2 \\ ImageNet-100} &
  \thead{ResNet-34 $\to$ ResNet-18 \\ TinyImageNet} &
  \thead{ResNet-50 $\to$ ResNet-18 \\ ImageNet-100} \\
\midrule
Teacher & 85.30 & 43.14 & 85.30 \\
Student & 80.21 & 42.98 & 83.85 \\
\midrule
CE only & 80.21 $\pm$ 0.07 & 42.98 $\pm$ 0.46 & 83.85 $\pm$ 0.15 \\
KD~\cite{hinton2015distilling} & 82.43 $\pm$ 0.41 ($\mathcolor{red}{\downarrow}$) & 44.36 $\pm$ 1.23 ($\mathcolor{red}{\downarrow}$) & 84.72 $\pm$ 0.35 ($\mathcolor{red}{\downarrow}$) \\
FitNets~\cite{romero2014fitnets} & 82.51 $\pm$ 0.30 ($\mathcolor{ForestGreen}{\uparrow}$) & 44.58 $\pm$ 0.83 ($\mathcolor{ForestGreen}{\uparrow}$) & 84.88 $\pm$ 0.22 ($\mathcolor{ForestGreen}{\uparrow}$) \\
AT~\cite{zagoruyko2016paying} & 82.90 $\pm$ 0.30 ($\mathcolor{ForestGreen}{\uparrow}$) & 44.42 $\pm$ 0.91 ($\mathcolor{ForestGreen}{\uparrow}$) & 84.92 $\pm$ 0.26 ($\mathcolor{ForestGreen}{\uparrow}$) \\
DKD~\cite{zhao2022decoupled} & 82.54 $\pm$ 0.18 ($\mathcolor{ForestGreen}{\uparrow}$) & 45.44 $\pm$ 1.11$^\dagger$ ($\mathcolor{ForestGreen}{\uparrow}$) & 84.85 $\pm$ 0.18 ($\mathcolor{ForestGreen}{\uparrow}$) \\
CRD~\cite{tian2019contrastive} & 83.01 $\pm$ 0.18 ($\mathcolor{ForestGreen}{\uparrow}$) & 45.02 $\pm$ 1.08 ($\mathcolor{ForestGreen}{\uparrow}$) & 85.24 $\pm$ 0.19 ($\mathcolor{ForestGreen}{\uparrow}$) \\
CkptKD~\cite{jin2022efficient} & 82.70 $\pm$ 0.30 ($\mathcolor{ForestGreen}{\uparrow}$) & 44.75 $\pm$ 1.02 ($\mathcolor{ForestGreen}{\uparrow}$) & 84.80 $\pm$ 0.28 ($\mathcolor{ForestGreen}{\uparrow}$) \\
\textbf{ASD (ours)} & \best{83.24 $\pm$ 0.13} ($\mathcolor{ForestGreen}{\uparrow}$) & \best{46.00 $\pm$ 1.45} ($\mathcolor{ForestGreen}{\uparrow}$) & \best{85.48 $\pm$ 0.15} ($\mathcolor{ForestGreen}{\uparrow}$) \\
\bottomrule
\end{tabular}
\vspace{4pt}
\caption{\small{Top-1 accuracy (\%) on ImageNet-100 and TinyImageNet. P11--P13 denote the teacher--student pairs shown in the header. $\mathcolor{ForestGreen}{\uparrow}$/$\mathcolor{red}{\downarrow}$ indicates above/below KD. Grey bold marks the best result per column. Results are mean\,$\pm$\,std over 3 seeds.}}
\label{tbl:beyond_cifar}
\end{center}
\vspace{-8pt}
\end{table}

\vspace{-3mm}
\Cref{tbl:beyond_cifar} tests whether the temporal shortcut signal remains
useful at higher resolution and class count. ASD obtains the best accuracy
on all 3 pairs ($83.24\%$ on P11, $46.00\%$ on P12, $85.48\%$ on P13).
The TinyImageNet result has notably larger seed variance ($\pm1.45$) than
the CIFAR-100 results, reflecting a less stable optimization landscape on
this dataset. The mean is still highest, and ASD wins on 3/3 seeds against
CRD with paired RNG initialization (per-seed margins $+0.62$, $+1.14$,
$+1.18$ pp; see supplementary for the seed-level analysis).

\vspace{-5mm}
\subsection{Semantic Segmentation on ADE20K}
\label{sec:segmentation}

\begin{table}[t]
\footnotesize
\centering
\setlength{\tabcolsep}{3pt}
\begin{tabular}{ll|c|c|c|c|c|c|c|c|c}
\toprule
\multirow{2}{*}{Teacher} & \multirow{2}{*}{Student}
  & \multicolumn{2}{c}{From Scratch}
  & \multicolumn{2}{c}{KD~\cite{hinton2015distilling}}
  & \multicolumn{2}{c}{Af-DCD~\cite{fan2023augmentation}}
  & \multicolumn{1}{c}{UFD-KD~\cite{lu2025ufd}}
  & \multicolumn{2}{c}{\textbf{ASD (ours)}} \\
\cmidrule(lr){3-4}\cmidrule(lr){5-6}\cmidrule(lr){7-8}\cmidrule(lr){9-9}\cmidrule(lr){10-11}
 & & T.\,mIoU & S.\,mIoU & pixAcc & mIoU & pixAcc & mIoU & mIoU & pixAcc & mIoU \\
\midrule
ResNet-101 & ResNet-18
  & 42.70 & 33.91
  & 76.32 & 34.88
  & 77.43 & 36.21
  & 36.89
  & 78.92 & \textbf{38.87} \\
ResNet-101 & MobileNetV2
  & 42.70 & 34.32
  & 76.74 & 34.92
  & 77.28 & 35.97
  & 36.24
  & 78.65 & \textbf{37.91} \\
\bottomrule
\end{tabular}

\caption{ADE20K semantic segmentation with DeepLabv3. Results are reported as pixAcc/mIoU (\%) using $512{\times}512$ inputs and 40k iterations. From-Scratch reports teacher/student mIoU; KD, Af-DCD~\cite{fan2023augmentation}, and UFD-KD~\cite{lu2025ufd} follow~\cite{lu2025ufd}. Bold denotes best mIoU per row.}

\vspace{-3mm}
\label{tbl:ade20k}
\end{table}

\vspace{-2mm}
To test whether the temporal shortcut signal transfers beyond
classification, we evaluate ASD on dense prediction. Both ASD losses act on
penultimate features, so they extend to segmentation without modification:
we apply $\Ltc$ and $\Lss$ to the backbone feature map of a

DeepLabv3~\cite{chen2018encoder} network, treating each spatial location as
a sample when forming the temporal displacement
$\Dh = \hearly - \hfinal$ and the uncentered shortcut covariance. We follow
the protocol of~\cite{lu2025ufd}: a ResNet-101 teacher distilled into
ResNet-18 and MobileNetV2 students, $512\times512$ inputs, 40k iterations,
SGD with momentum $0.9$ and learning rate $0.02$, reporting pixel accuracy
and mean IoU on the ADE20K~\cite{zhou2017scene} validation set.

\Cref{tbl:ade20k} reports the results. ASD improves both students over the
KD baseline, the dense contrastive method Af-DCD~\cite{fan2023augmentation},
and the recent frequency-decoupling method UFD-KD~\cite{lu2025ufd}, reaching
$38.87$ mIoU for ResNet-18 ($+1.98$ over UFD-KD, $+3.99$ over the from-scratch
student) and $37.91$ mIoU for MobileNetV2 ($+1.67$ over UFD-KD). The gain is
consistent with the mechanistic picture from classification: dense
prediction is especially exposed to texture and local-statistic shortcuts,
so suppressing the early-teacher displacement subspace and contrasting
against the early teacher state transfers the same shape-biased,
shortcut-resistant structure to the per-pixel features. ASD requires no
segmentation-specific component beyond pooling the two frozen teacher
checkpoints' feature maps, which keeps the added cost in line with the
classification setting. t-SNE and qualitative mask visualizations
(Supplementary Section A.10) corroborate these gains: the ASD student forms
tighter, better-separated per-class clusters and recovers sharper object
boundaries than the KD student.

\vspace{-5mm}
\subsection{Corruption Robustness}
\label{sec:robustness}
\vspace{-3mm}
\begin{table}[ht]
\footnotesize
\setlength{\tabcolsep}{3pt}
\begin{center}
\resizebox{\linewidth}{!}{
\begin{tabular}{l|c|c|c|c|c|c}
\toprule
\thead{Teacher \\ Student} & \thead{ResNet-34 \\ ResNet-18} & \thead{WRN-40-2 \\ WRN-16-2} & \thead{ResNet-32$\times$4 \\ ResNet-8$\times$4} & \thead{WRN-40-2 \\ WRN-40-1} & \thead{VGG-13 \\ VGG-8} & \thead{WRN-40-2 \\ ShuffleNet-V2} \\
\midrule
CE only & 87.0 $\pm$ 0.3 ($\mathcolor{red}{\uparrow}$) & 95.9 $\pm$ 0.2 ($\mathcolor{red}{\uparrow}$) & 99.6 $\pm$ 0.4 ($\mathcolor{red}{\uparrow}$) & 93.7 $\pm$ 0.4 ($\mathcolor{red}{\uparrow}$) & 87.0 $\pm$ 0.3 ($\mathcolor{red}{\uparrow}$) & 88.5 $\pm$ 0.4 ($\mathcolor{red}{\uparrow}$) \\
KD~\cite{hinton2015distilling} & 86.2 $\pm$ 0.4 ($\mathcolor{red}{\uparrow}$) & 93.7 $\pm$ 0.4 ($\mathcolor{red}{\uparrow}$) & 98.1 $\pm$ 0.3 ($\mathcolor{red}{\uparrow}$) & 92.2 $\pm$ 0.6 ($\mathcolor{red}{\uparrow}$) & 85.6 $\pm$ 0.3 ($\mathcolor{red}{\uparrow}$) & 86.9 $\pm$ 0.4 ($\mathcolor{red}{\uparrow}$) \\
FitNets~\cite{romero2014fitnets} & 86.0 $\pm$ 0.4 ($\mathcolor{ForestGreen}{\downarrow}$) & 93.6 $\pm$ 0.1 ($\mathcolor{ForestGreen}{\downarrow}$) & 98.0 $\pm$ 0.3 ($\mathcolor{ForestGreen}{\downarrow}$) & 92.1 $\pm$ 0.3 ($\mathcolor{ForestGreen}{\downarrow}$) & 85.5 $\pm$ 0.2 ($\mathcolor{ForestGreen}{\downarrow}$) & 86.8 $\pm$ 0.3 ($\mathcolor{ForestGreen}{\downarrow}$) \\
AT~\cite{zagoruyko2016paying} & 85.7 $\pm$ 0.6 ($\mathcolor{ForestGreen}{\downarrow}$) & 93.6 $\pm$ 0.5 ($\mathcolor{ForestGreen}{\downarrow}$) & 97.8 $\pm$ 0.5 ($\mathcolor{ForestGreen}{\downarrow}$) & 92.0 $\pm$ 0.4 ($\mathcolor{ForestGreen}{\downarrow}$) & 85.5 $\pm$ 0.3 ($\mathcolor{ForestGreen}{\downarrow}$) & 87.4 $\pm$ 0.2 ($\mathcolor{red}{\uparrow}$) \\
DKD~\cite{zhao2022decoupled} & 85.9 $\pm$ 0.9 ($\mathcolor{ForestGreen}{\downarrow}$) & \best{93.4 $\pm$ 0.3} ($\mathcolor{ForestGreen}{\downarrow}$) & 97.5 $\pm$ 0.2 ($\mathcolor{ForestGreen}{\downarrow}$) & 92.4 $\pm$ 0.3 ($\mathcolor{red}{\uparrow}$) & 85.5 $\pm$ 0.1 ($\mathcolor{ForestGreen}{\downarrow}$) & 87.4 $\pm$ 0.5 ($\mathcolor{red}{\uparrow}$) \\
CRD~\cite{tian2019contrastive} & \best{85.2 $\pm$ 0.6} ($\mathcolor{ForestGreen}{\downarrow}$) & 93.8 $\pm$ 0.3 ($\mathcolor{red}{\uparrow}$) & 98.2 $\pm$ 0.1 ($\mathcolor{red}{\uparrow}$) & 92.7 $\pm$ 0.6 ($\mathcolor{red}{\uparrow}$) & 85.9 $\pm$ 0.5 ($\mathcolor{red}{\uparrow}$) & 86.7 $\pm$ 0.3 ($\mathcolor{ForestGreen}{\downarrow}$) \\
CkptKD~\cite{jin2022efficient} & 90.9 $\pm$ 0.1 ($\mathcolor{red}{\uparrow}$) & 99.1 $\pm$ 0.2 ($\mathcolor{red}{\uparrow}$) & \best{96.6 $\pm$ 0.2} ($\mathcolor{ForestGreen}{\downarrow}$) & 98.7 $\pm$ 0.2 ($\mathcolor{red}{\uparrow}$) & 90.0 $\pm$ 0.2 ($\mathcolor{red}{\uparrow}$) & 92.2 $\pm$ 0.1 ($\mathcolor{red}{\uparrow}$) \\
\textbf{ASD (ours)} & 85.6 $\pm$ 0.4 ($\mathcolor{ForestGreen}{\downarrow}$) & 94.2 $\pm$ 0.4 ($\mathcolor{red}{\uparrow}$) & 97.9 $\pm$ 0.2 ($\mathcolor{ForestGreen}{\downarrow}$) & \best{91.7 $\pm$ 0.4} ($\mathcolor{ForestGreen}{\downarrow}$) & \best{85.2 $\pm$ 0.3} ($\mathcolor{ForestGreen}{\downarrow}$) & \best{86.1 $\pm$ 0.4} ($\mathcolor{ForestGreen}{\downarrow}$) \\
\bottomrule
\end{tabular}
}
\vspace{-4pt}
\caption{\small{Mean Corruption Error (mCE, \%) on CIFAR-100-C at severity 5. Lower is better. $\mathcolor{ForestGreen}{\downarrow}$\,/\,$\mathcolor{red}{\uparrow}$ indicates improvement/degradation vs.\ KD. Grey bold marks the lowest mCE per column. Results are mean\,$\pm$\,std over 3 seeds.}}

\label{tbl:robustness}
\end{center}
\vspace{-8pt}
\end{table}

Robustness (\cref{tbl:robustness}) tests whether reduced shortcut projection translates into reduced sensitivity to corrupted local
statistics. ASD obtains the lowest mCE on P4 ($91.7$), P5 ($85.2$), and
the most challenging cross-architecture pair P6 ($86.1$). These are pairs
where ASD also performs strongly on clean accuracy, suggesting that
shortcut suppression does not simply trade accuracy for robustness; when
the student has sufficient capacity, it redirects the representation
toward features that improve both.
\vspace{-5mm}
\paragraph{P1: narrower trade-off.}
On P1, CRD has the lowest mCE ($85.2$); ASD reaches $85.6$, still
improving over KD ($86.2$). The two reduce error differently: CRD
strengthens instance-level geometry, ASD penalizes temporal shortcut
directions.
\vspace{-5mm}
\paragraph{Structural failure of CkptKD.}
CkptKD degrades robustness severely on five of six pairs (P1:
$86.2\to90.9$; P2: $93.7\to99.1$; P6: $86.9\to92.2$). The reason is
structural: CkptKD treats an early checkpoint as an \emph{attractive}
teacher, codifying its texture- and background-biased features. ASD uses
the same source in the opposite direction and obtains the opposite effect.
\vspace{-5mm}
\paragraph{Honest negative: P2.}
On WRN-40-2$\to$WRN-16-2, ASD reaches $94.2$ mCE, above KD ($93.7$) and
DKD ($93.4$): the dominant error mode here is class-boundary calibration
rather than shortcut reliance, where DKD's logit separation is more
effective (per-category decomposition in Supplementary Sec.~A.9.2). We
report this without exclusion.

\vspace{-5mm}
\subsection{The P3 Capacity Boundary}
\label{sec:p3}
\vspace{-2mm}
P3 (ResNet-32$\times$4$\to$ResNet-8$\times$4) is the only same-family
pair where ASD trails KD on clean accuracy, and the only pair where the
negative-source ablation (\cref{sec:ablation_negative_source}) prefers
in-batch negatives. ResNet-8$\times$4 has $6\times$ fewer parameters than
its teacher, and for such aggressive compression some low-level cues
function as necessary discriminative scaffolding rather than dispensable
shortcuts. This is the practical boundary of ASD's applicability: shortcut
suppression is most reliable when the student can \emph{replace} shortcut
cues with semantic structure. Consistently, CkptKD which encourages
retention of early-teacher features has its only robustness
improvement on P3, (Supplementary Sec.~A.9.1).

\vspace{-6mm}
\section{Ablation Studies}
\label{sec:ablations}

\begin{table*}[t]
\centering
\footnotesize
\setlength{\tabcolsep}{4pt}
\renewcommand{\arraystretch}{1.05}

\begin{minipage}[t]{0.47\textwidth}
\vspace{0pt}
\centering
\begin{tabular}{l|cc}
\toprule
\textbf{Method}
& \textbf{Clean Acc (\%)} $\uparrow$
& \textbf{mCE (\%)} $\downarrow$ \\
\midrule
KD~\cite{hinton2015distilling}
& $80.22 \pm 0.19$ (${\color{red}\downarrow}$)
& $86.2 \pm 0.4$ ({\color{red}$\uparrow$}) \\
ASD w/o $\mathcal{L}_{SS}$
& $80.19 \pm 0.07$ ({\color{red}$\downarrow$})
& $85.8 \pm 0.7$ ({\color{green!50!black}$\downarrow$}) \\
ASD w/o $\mathcal{L}_{TC}$
& $80.15 \pm 0.19$ ({\color{red}$\downarrow$})
& $85.9 \pm 0.4$ ({\color{green!50!black}$\downarrow$}) \\
\rowcolor{gray!15}
\textbf{ASD (ours)}
& $\mathbf{80.50 \pm 0.07}$ ({\color{green!50!black}$\uparrow$})
& $\mathbf{85.6 \pm 0.4}$ ({\color{green!50!black}$\downarrow$}) \\
\bottomrule
\end{tabular}

\vspace{1mm}
\caption{Loss ablation on P1. Removing $\mathcal{L}_{SS}$ keeps only temporal contrast, while removing $\mathcal{L}_{TC}$ keeps only shortcut suppression.}
\label{tbl:ablation_components}
\end{minipage}
\hfill
\begin{minipage}[t]{0.47\textwidth}
\vspace{0pt}
\centering
\begin{tabular}{c|c|c}
\toprule
\textbf{$T_{\mathrm{early}}$ epoch}
& \textbf{\% training}
& \textbf{Clean Acc (\%)} \\
\midrule
12 & 5\%  & 80.25 \\
24 & 10\% & 80.26 \\
\rowcolor{gray!15}
36 & 15\% & $\mathbf{80.50 \pm 0.07}$ $\leftarrow$ default \\
48 & 20\% & 80.16 \\
72 & 30\% & 80.56 \\
\bottomrule
\end{tabular}

\vspace{1mm}
\caption{Sensitivity of ASD to the early checkpoint epoch $T_{\mathrm{early}}$ on P1. Epoch 36 is adopted as the default setting.}
\label{tbl:ablation_tearly}
\end{minipage}

\vspace{-2mm}
\end{table*}

\vspace{-2mm}

\subsection{Loss Components}
\label{sec:ablation_components}

\Cref{tbl:ablation_components} decomposes ASD into temporal contrast and
shortcut suppression on P1. Removing $\Lss$ (leaving $\Ltc$ only) gives
$80.19\%$ clean accuracy with $85.8$ mCE; removing $\Ltc$ (leaving $\Lss$
only) gives $80.15\%$ and $85.9$ mCE; the full model reaches $80.50\%$ and
$85.6$ mCE. The pattern supports the intended division of labour: $\Ltc$
mainly improves discriminative transfer by making the final teacher more
probable than the early teacher in representation space, while $\Lss$ acts
more directly on robustness by constraining the student's projection onto
shortcut eigendirections. Either component alone improves part of the
objective, but neither fully reconstructs the push--pull geometry. Their
combination is not redundant: one term defines the temporal contrast, the
other enforces a geometric exclusion region.

\vspace{-3mm}

\subsection{\texorpdfstring{$\Tearly$}{T\_early} Sensitivity}
\label{sec:ablation_tearly}

\Cref{tbl:ablation_tearly} shows ASD is stable across the $10\%$--$30\%$
early-to-middle range: epoch 36 (fixed 15\% default) gives
$80.50\pm0.07$, epoch 72 gives $80.56\%$, and only epoch 12 is weaker. This
stability makes the default practical: $\Tearly$ is a zero-cost
byproduct of training $\Tfinal$ (Supplementary Sec.~A.8), and if only
$\Tfinal$ is available, a 10-epoch high-LR fine-tune gives an early-like
reference ($80.18\%$ on P1), close to KD ($80.22\%$) and the true early
checkpoint ($80.27\%$; Supplementary Sec.~A.4.6).

\vspace{-3mm}
\subsection{Negative Source for \texorpdfstring{$\Ltc$}{L\_TC}}
\label{sec:ablation_negative_source}
\vspace{-3mm}
\begin{table}[ht]
\footnotesize
\setlength{\tabcolsep}{4pt}
\begin{center}
\resizebox{\linewidth}{!}{
\begin{tabular}{ll|c|c|c|c|c}
\toprule
\thead{Negative \\ Source} & \thead{Objective} &
  \thead{ResNet-34 \\ ResNet-18} &
  \thead{WRN-40-2 \\ WRN-16-2} &
  \thead{ResNet-32$\times$4 \\ ResNet-8$\times$4} &
  \thead{WRN-40-2 \\ ShuffleNet-V2} &
  \thead{VGG-13 \\ VGG-8} \\
\midrule
In-batch ($i \neq j$) & CE + KD 
& 80.22 $\pm$ 0.19 
& 74.74 $\pm$ 0.17 
& \best{73.32 $\pm$ 0.11} 
& 66.72 $\pm$ 0.32 
& 73.68 $\pm$ 0.35 \\
$T_{\rm early}$ (ours) & ASD 
& \best{80.50 $\pm$ 0.07} 
& \best{75.87 $\pm$ 0.29} 
& 73.13 $\pm$ 0.12 
& \best{68.05 $\pm$ 0.14} 
& \best{74.26 $\pm$ 0.13} \\
\bottomrule
\end{tabular}
}
\caption{\small{Negative-source ablation for $\mathcal{L}_{\rm TC}$ on five CIFAR-100 pairs. $T_{\rm early}$ provides structured temporal negatives and improves 4/5 pairs; P3 reflects low shortcut reliance. Grey bold marks the best result per column. Results are mean\,$\pm$\,std over 3 seeds.}}
\label{tbl:ablation_negative_source}
\end{center}
\vspace{-12pt}
\end{table}

\Cref{tbl:ablation_negative_source} most directly defends the novelty
claim against CRD. Replacing the same-sample temporal negative with random
in-batch negatives reduces $\Ltc$ to a CRD-style objective and removes the
teacher trajectory signal entirely. On 4 of 5 pairs, using $\Tearly$ as
the structured negative outperforms in-batch negatives, with the largest
gain on the cross-architecture pair WRN-40-2$\to$ShuffleNet-V2
($+1.33$ pp).
The exception is again P3, consistent with the capacity boundary discussed
in \cref{sec:p3}. The pattern supports the interpretation that
cross-architecture students benefit most from knowing not only what to
imitate, but also which teacher state to avoid.

\begin{figure}[t]
  \centering
  \includegraphics[
    width=0.98\columnwidth,
  ]{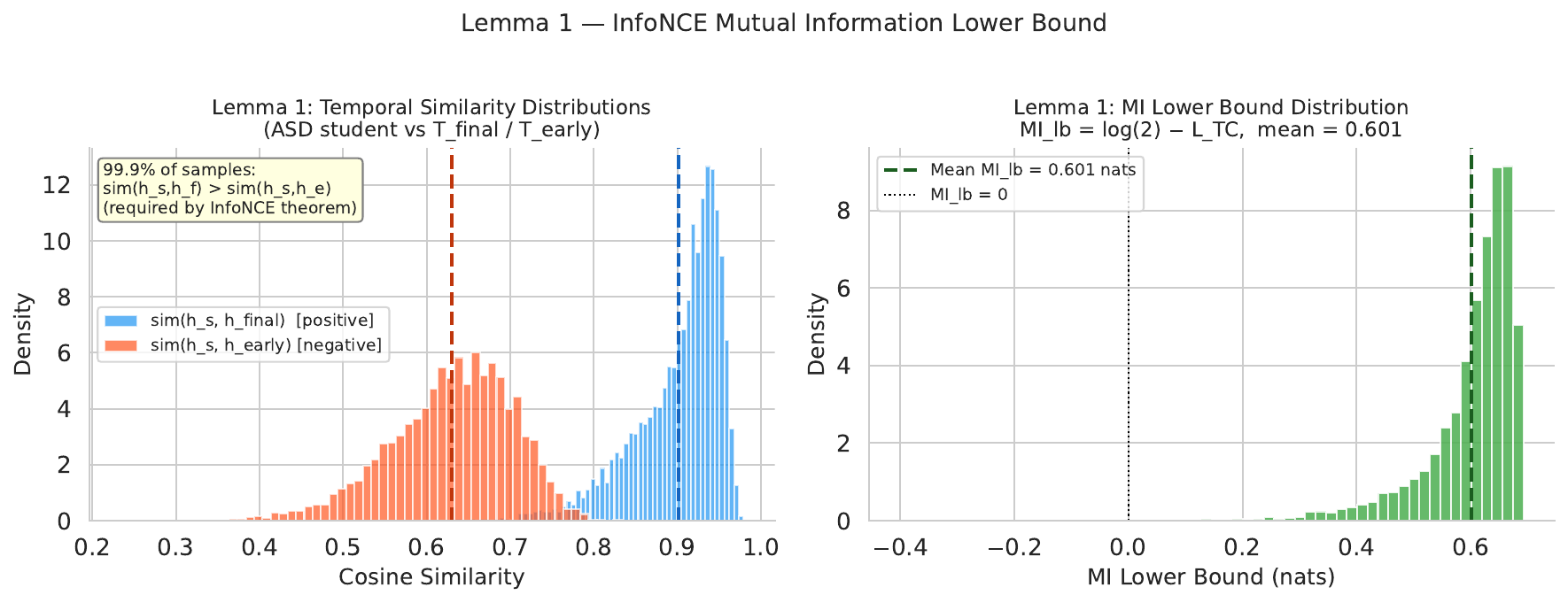}
  {\footnotesize\textbf{(a)} Temporal alignment ($\Ltc$)}
  \includegraphics[
    width=0.9\columnwidth,
  ]{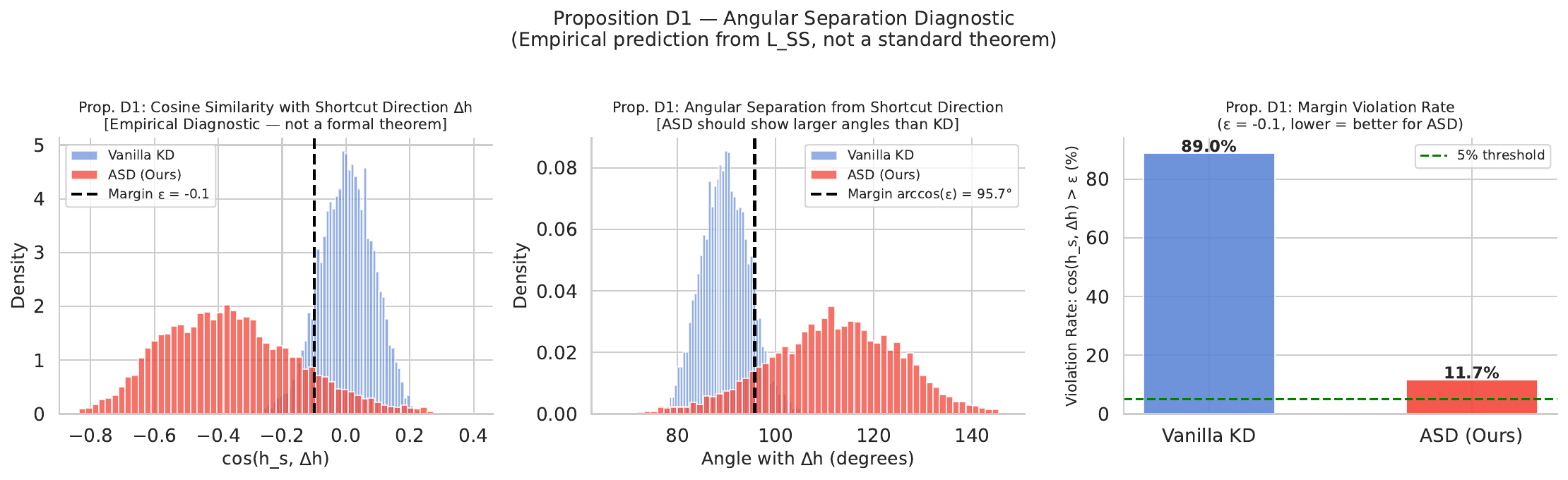}
  
  {\footnotesize\textbf{(b)} Shortcut-direction separation ($\Lss$)}
  \caption{\emph{Alignment and shortcut-direction diagnostics on P1.}
  ASD aligns the student with $\Tfinal$ over $\Tearly$ for $99.9\%$ of samples
  with an InfoNCE bound of $0.601$ nats, and produces a negative shortcut cosine
  compared with KD ($-0.361$ vs.\ $0.003$).}
  \label{fig:mech_alignment_shortcut}
  \vspace{-8pt}
\end{figure}


\vspace{-5mm}
\subsection{Mechanistic Validation}
\label{sec:mechanistic_validation}
We verify that the learned student actually follows the geometry the
method prescribes. All diagnostics use P1 (ResNet-34$\to$ResNet-18 on
CIFAR-100); full procedures, bootstrap stability, and the texture-shift
protocol are in Supplementary Sec.~A.3.
\Cref{fig:mech_alignment_shortcut} and \cref{fig:mech_subspace_robustness}
consolidate the four core diagnostics. The mechanism is not pair-specific: Supplementary Sec.~A.3.7 repeats
the diagnostics on P5, P3, and P12, in addition to P1. ASD reproduces the
same signature on every pair: higher $\Tfinal$ than $\Tearly$ alignment
(e.g., $0.90/0.61$ on P1 and $0.52/0.34$ on P12) and reduced shortcut
magnitude ($-16$ to $-28\%$), with CIFAR-100-C gains concentrated on
texture/local-statistic corruptions.

\vspace{-4mm}
\paragraph{Temporal alignment ($\Ltc$ diagnostic,
\cref{fig:mech_alignment_shortcut}a):}
At convergence, $99.9\%$ of samples satisfy
$\mathrm{sim}(\hs,\hfinal) > \mathrm{sim}(\hs,\hearly)$, confirming that
$\Ltc$ specifically moves the student closer to the converged teacher than
to the early teacher. The binary-diagnostic InfoNCE lower bound is $0.601$
nats, close to its hard binary ceiling of $\log 2 \approx 0.693$ nats; we
read this as a directional indicator that the student systematically
prefers $\Tfinal$, not as an absolute mutual-information claim. A
non-bound, full-set diagnostic is in
Supplementary Sec.~A.3.6.

\vspace{-4mm}
\paragraph{Anti-alignment with the shortcut direction
($\Lss$ diagnostic, \cref{fig:mech_alignment_shortcut}b):}
The signed cosine $\cos(\hs,\Dh)$ has mean $-0.361$ for ASD versus $0.003$
for KD. Under the diagnostic threshold $\varepsilon_{\mathrm{diag}} =
-0.1$, $89.0\%$ of KD samples fall on the shortcut-aligned side, against
$11.7\%$ for ASD. ASD does not merely fail to align with $\Dh$ it is
\emph{anti}-aligned; the representational signature predicted by $\Ltc$.

\vspace{-5mm}
\paragraph{Subspace geometry (\cref{thm:shortcut_id} diagnostic,
\cref{fig:mech_subspace_robustness}a):}
The principal angle between the shortcut subspace
$\mathrm{span}(\Uk)$ and the $\Tfinal$-PCA robust subspace is
$65.3^{\circ}$, against $38.6^{\circ}$ to the $\Tearly$-PCA subspace.
The shortcut subspace lies substantially closer to the early
teacher than to the final teacher, consistent with 
\cref{thm:shortcut_id,eq:shortcut_energy}. The ASD student's projection
onto the robust subspace has mean magnitude $0.45$ versus $0.12$ for KD:
ASD does not merely preserve the robust subspace, it re-allocates capacity
toward it, explaining why clean accuracy and robustness improve together.

\vspace{-5mm}
\paragraph{Texture-shift diagnostic (\cref{fig:mech_subspace_robustness}b):}
Under a stylized-input perturbation following~\cite{geirhos2019imagenet},
ASD obtains $80.3\%$ clean and $49.1\%$ accuracy (degradation
$-31.12\%$), versus $80.2\% / 48.1\%$ for KD ($-32.11\%$). The gap is
modest but directionally consistent: reduced shortcut projection reduces
sensitivity to texture perturbation without eliminating it.

\vspace{-5mm}
\paragraph{Subspace stability: }
The shortcut covariance has a small eigengap ($0.059$ absolute, $0.029$
relative), so individual eigenvectors are not stably oriented; however,
bootstrap resampling gives an inter-resample principal angle of
$32.6^{\circ} \pm 1.6^{\circ}$   the \emph{region} of shortcut activation
is reproducible even though the eigenbasis is not. As $\Lss$ depends on
subspace projection, not eigenvector orientation, this is sufficient
(Supplementary Sec.~A.7).

\begin{figure}[t]
  \centering

  \includegraphics[
    width=\linewidth,
    trim=1mm 2mm 1mm 18mm,
    clip
  ]{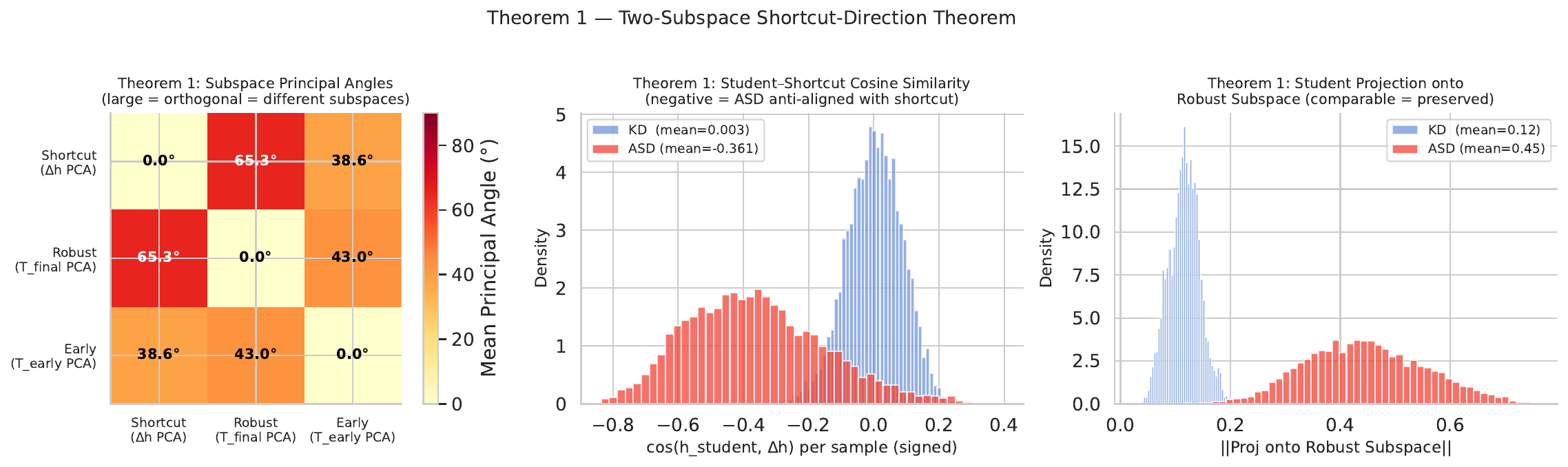}

  \vspace{-2pt}
  {\footnotesize\textbf{(a)} Subspace geometry (\cref{thm:shortcut_id}) \label{}}


  \includegraphics[
    width=0.62\linewidth,
    trim=1mm 2mm 1mm 18mm,
    clip
  ]{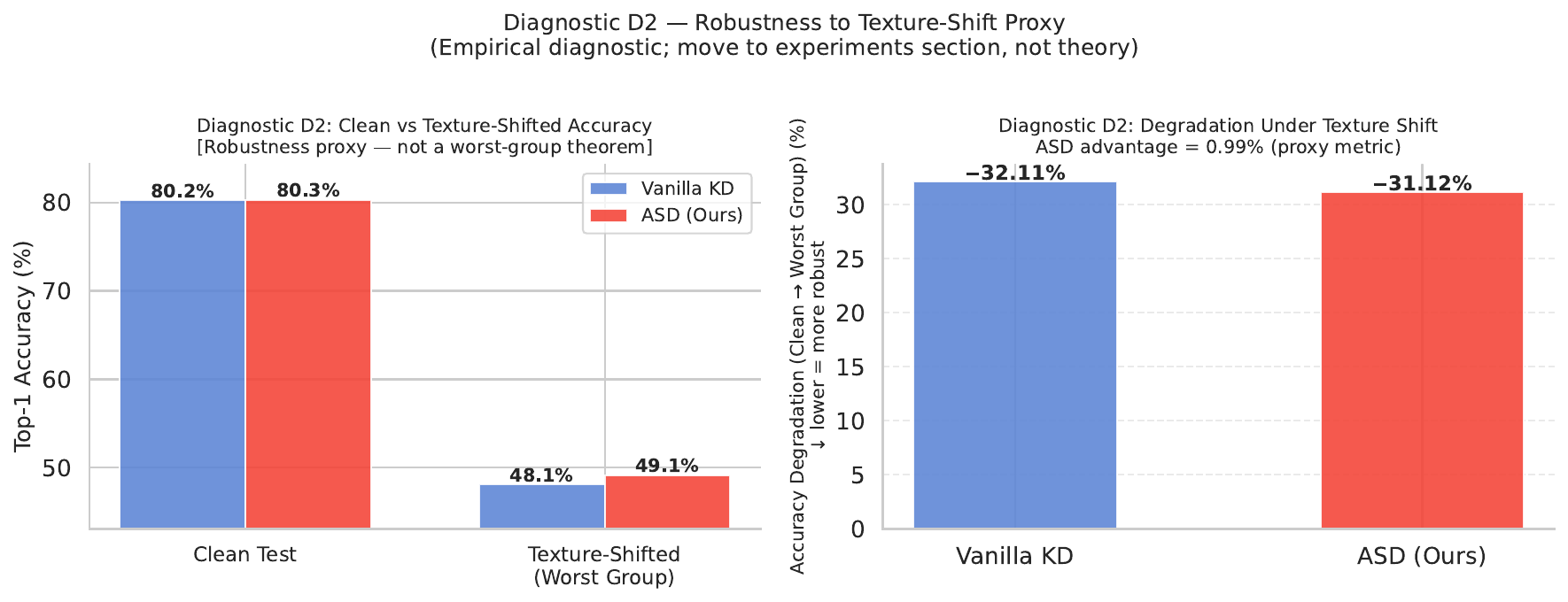}

  {\footnotesize\textbf{(b)} Texture-shift robustness} \label{Fig:texture-shift}

  \vspace{-2pt}
\caption{\emph{Subspace and robustness diagnostics on P1.}
ASD separates shortcut directions, improves robust-subspace projection over KD ($0.45$ vs.\ $0.12$), and slightly reduces texture-shift degradation ($-31.12\%$ vs.\ $-32.11\%$).}
  \label{fig:mech_subspace_robustness}

  \vspace{-12pt}
\end{figure}
\vspace{-5mm}
\section{Conclusion}
\label{sec:conclusion}
\vspace{-4mm}
We presented Anti-Shortcut Distillation, which converts a teacher's
training trajectory into both attractive and repulsive supervision. ASD treats $\Tfinal$ as a positive anchor and $\Tearly$ as a temporal
negative; $\Ltc$ and $\Lss$ align the student with converged teacher
features while suppressing projection onto shortcut-prone directions. Across 13
teacher--student pairs, ASD attains the best clean accuracy on 10 and
improves over KD on 12, obtains the lowest mCE on the hardest
cross-architecture pair ($86.1$), and is mechanistically anti-aligned with
the shortcut direction ($\cos(\hs,\Dh) = -0.36$ vs.\ $0.00$ for KD) while
tripling robust-subspace projection ($0.45$ vs.\ $0.12$).
\vspace{-5mm}
\paragraph{Limitations and future work:} ASD stores one extra frozen checkpoint and computes a per-batch
eigendecomposition, a small marginal cost. P3 highlights a capacity
boundary: severely compressed students may still need low-level cues for
discrimination. A transformer-student result (DeiT-Tiny, Supplementary Sec.~A.4.7) suggests the signal is not CNN-specific; large-model and self-supervised
validation remain open alongside a
tighter $\Ltc$--$\Lss$ characterization.

\section*{Acknowledgment}
This work was supported by Institute of Information and communications Technology Planning and Evaluation (IITP) grant funded by the Korea government (MSIT) (No. RS-2024-00443391, Korea-Japan Joint International Research on Deep Learning-based Autonomous Mobility Control Technology for Autonomous Mobile Robots).

\bibliography{bib_additions}
\end{document}


\maketitle

\noindent This supplementary material contains:
\cref{sec:supp_proofs}, full proofs of the propositions and theorem stated
in the main paper;
\cref{sec:supp_extended_theory}, an extended theoretical analysis
(gradient geometry of both losses and a first-order robustness bound);
\cref{sec:supp_validation}, implementation and procedure for the
mechanistic diagnostics;
\cref{sec:supp_extra_ablations}, additional hyperparameter ablations
(subspace dimension $K$, margin $\varepsilon$, memory size $M$, warmup
duration);
\cref{sec:supp_repro}, the full reproduction protocol;
\cref{sec:supp_significance}, a paired-seed significance analysis for
TinyImageNet;
\cref{sec:supp_subspace_stability}, an extended discussion of Davis--Kahan
subspace stability;
\cref{sec:supp_complexity}, computational and memory complexity; and
\cref{sec:supp_neg_results}, an extended analysis of negative results; and
\cref{sec:supp_ade20k_vis}, ADE20K t-SNE and qualitative-mask visualizations.
Further material added during the rebuttal is marked in blue: multi-pair
diagnostics and a per-corruption breakdown (Sec.~A.3.7), the teacher
feature-energy trajectory (Sec.~A.3.8, Fig.~A.6), targeted shortcut
benchmarks with per-seed results (Sec.~A.3.9), a batch-size ablation
(Sec.~A.4.5), the rewound-reference experiment (Sec.~A.4.6), and a
transformer student (Sec.~A.4.7).
We do not restate the main-paper definitions; cross-references use the
labels from the main paper.

\section{Proofs}
\label{sec:supp_proofs}

\subsection{Proof of Proposition~1 (InfoNCE Alignment) of the Main Paper}
\label{sec:supp_proof_infonce}

\emph{Restated formally.} Let
$(\hs(x), \hfinal(x)) \sim p_{S,F}$ be drawn from the joint distribution
over positives, and let $\mathcal{N}_F = \{f_1, \ldots, f_N\}$ with
$f_n \stackrel{\mathrm{iid}}{\sim} p_{\hfinal}$, independent of the
positive pair. Define the population reduced loss
\begin{equation}
  \widetilde{\Ltc}(x; \mathcal{N}_F) \;=\; -\log
  \frac{ \exp\!\left( \hs(x)^{\top} \hfinal(x) / \tau_c \right) }{
  \exp\!\left( \hs(x)^{\top} \hfinal(x) / \tau_c \right) +
  \sum_{n=1}^{N} \exp\!\left( \hs(x)^{\top} f_n / \tau_c \right) }.
  \label{eq:supp_ltc_reduced}
\end{equation}
Then~\cite{oord2018representation,poole2019variational}:
\begin{equation}
  I(\hs; \hfinal) \;\geq\; \log(N+1) - \E_{x, \mathcal{N}_F}
  \left[ \widetilde{\Ltc}(x; \mathcal{N}_F) \right].
  \label{eq:supp_infonce_bound}
\end{equation}

\begin{proof}
We follow the standard
derivation~\cite{oord2018representation,poole2019variational}.
Consider $N+1$ samples: the true positive $\hfinal(x)$ drawn from
$p(\hfinal \mid \hs(x))$, and $N$ negatives drawn from $p(\hfinal)$. The
Bayes-optimal classifier identifying the positive index assigns probability
\begin{equation}
  p^{*}(j \mid \hs(x), \{f_i\}_{i=0}^{N}) \;=\;
  \frac{ p(f_j \mid \hs(x)) / p(f_j) }{
  \sum_{i=0}^{N} p(f_i \mid \hs(x)) / p(f_i) },
\end{equation}
where $f_0 = \hfinal(x)$ is the true positive at index $0$. Using the
parametric model
$\exp(\hs(x)^{\top} f / \tau_c) \propto p(f \mid \hs(x)) / p(f)$, the
softmax in \cref{eq:supp_ltc_reduced} approximates $p^{*}$. The cross-entropy
of the indicator at index $0$ against $p^{*}$ is then
$-\log p^{*}(0) = \widetilde{\Ltc}$ in expectation. Standard
arguments~\cite{oord2018representation}, Theorem~1 give
\begin{equation}
  \E[\widetilde{\Ltc}] \;\geq\; -I(\hs;\hfinal) + \log(N+1),
\end{equation}
which rearranges to \cref{eq:supp_infonce_bound}.

\emph{Adding the temporal negative.}
$\Ltc$ adds the extra denominator term
$\exp(\hs(x)^{\top}\hearly(x)/\tau_c)$ inside the softmax in
\cref{eq:supp_ltc_reduced}. For any non-negative term added to the
denominator of a softmax, the resulting log-probability of the
positive class decreases, i.e.\ $\Ltc \geq \widetilde{\Ltc}$ pointwise.
Taking expectations preserves the inequality:
$\E[\Ltc] \geq \E[\widetilde{\Ltc}]$. Substituting into
\cref{eq:supp_infonce_bound} gives
\begin{equation}
  I(\hs; \hfinal) \;\geq\; \log(N+1) - \E[\widetilde{\Ltc}]
  \;\geq\; \log(N+1) - \E[\Ltc].
\end{equation}
\end{proof}

\noindent\emph{Discussion of bound tightness.} The InfoNCE bound has a
known saturation gap of up to $\log(N+1)$ nats; with $N=4159$ for our
default setting ($B-1=63$ in-batch plus $M=4096$ memory negatives), the
ceiling is $\log(4160) \approx 8.33$ nats, so the bound is uninformative
for absolute MI quantification but useful as a relative monotonic
indicator. We emphasise that minimizing $\Ltc$ raises the bound and
simultaneously \emph{decreases} similarity to $\Tearly$ through the extra
denominator term, at no cost to the bound. For non-bound MI estimators
(e.g.\ MINE~\cite{belghazi2018mine}), see \cref{sec:supp_validation}.

\subsection{Proof of Proposition~2 (Angular Margin) of the Main Paper}
\label{sec:supp_proof_margin}

\begin{proof}
Let $\hSbar(x_i)$ be unit-norm and let $\Uk = [u_1,\dots,u_K]$ have
orthonormal columns spanning the shortcut subspace $\mathcal{U} :=
\mathrm{span}(\Uk)$. The orthogonal projector onto $\mathcal{U}$ is
$P = \Uk\Uk^{\top}$, and the projection of $\hSbar(x_i)$ onto $\mathcal{U}$
is $P\hSbar(x_i)$. Because $\Uk$ has orthonormal columns,
\begin{equation}
  \norm{P\hSbar(x_i)}_2^2
  = \hSbar(x_i)^{\top}\Uk\Uk^{\top}\Uk\Uk^{\top}\hSbar(x_i)
  = \hSbar(x_i)^{\top}\Uk\Uk^{\top}\hSbar(x_i)
  = \norm{\Uk^{\top}\hSbar(x_i)}_2^2,
\end{equation}
using $\Uk^{\top}\Uk = I_K$. Hence
$\norm{\hSbar(x_i)^{\top}\Uk}_2 = \norm{P\hSbar(x_i)}_2$.
Let $\theta(x_i)\in[0,\tfrac{\pi}{2}]$ denote the angle between
$\hSbar(x_i)$ and its projection $P\hSbar(x_i)$ (the nearest point of
$\mathcal{U}$ to $\hSbar(x_i)$). Since $\hSbar(x_i)$ is unit-norm,
$\norm{P\hSbar(x_i)}_2 = \cos\theta(x_i)$ \emph{exactly}. Therefore
\begin{equation}
  \norm{\hSbar(x_i)^{\top}\Uk}_2 \;=\; \cos\theta(x_i).
  \label{eq:supp_proj_cos}
\end{equation}
By construction of the hinge, $\Lss(x_i)=0$ iff
$\norm{\hSbar(x_i)^{\top}\Uk}_2 \le \varepsilon$. Substituting
\cref{eq:supp_proj_cos} gives $\cos\theta(x_i) \le \varepsilon$, and since
$\arccos$ is decreasing on $[0,1]$, $\theta(x_i)\ge\arccos(\varepsilon)$.
For $\varepsilon = 0.1$, $\arccos(0.1) = 84.26^{\circ}$.
\end{proof}

\noindent\emph{Remark (relation to the diagnostic threshold).} The training
margin $\varepsilon$ acts on the \emph{unsigned} subspace-projection
magnitude $\norm{\hSbar^{\top}\Uk}_2 \in [0,1]$, so any value
$\varepsilon\in(0,1)$ defines a genuine angular margin. The diagnostic
threshold $\varepsilon_{\mathrm{diag}}=-0.1$ used in
\cref{sec:supp_validation} is a different object: it thresholds the
\emph{signed} cosine $\cos(\hs,\Dh)\in[-1,1]$ against a single direction
$\Dh$, and its negative sign reflects that a successfully trained ASD
student is \emph{anti}-aligned with $\Dh$. The two should not be conflated.

\subsection{Proof of Theorem~4 (Shortcut Subspace Identification) of the Main Paper}
\label{sec:supp_proof_thm1}

\textbf{Step 1: Exact decomposition under orthogonality.}
Write $\Dh(x) = \Dh_r(x) + \Dh_s(x)$, where
$\Dh_r(x) = r_{\mathrm{early}}(x) - r_{\mathrm{final}}(x) \in \Vrob$ and
$\Dh_s(x) = s_{\mathrm{early}}(x) - s_{\mathrm{final}}(x) \in \Vsc$. Choose
orthonormal bases $\{e_1, \ldots, e_p\}$ of $\Vrob$ and
$\{e_{p+1}, \ldots, e_{d_T}\}$ of $\Vsc$, where $p = \dim(\Vrob)$. In this
basis, $\Dh_r$ has nonzero coordinates only in positions $1, \ldots, p$,
and $\Dh_s$ has nonzero coordinates only in positions $p+1, \ldots, d_T$.

Expanding $\E[\Dh\Dh^{\top}] = \E[(\Dh_r + \Dh_s)(\Dh_r + \Dh_s)^{\top}]$
and writing
$\E[\Dh_r\Dh_s^{\top}]_{ij}$ for the $(i,j)$ entry,
we obtain
\begin{equation}
  \E[\Dh_r\Dh_s^{\top}]_{ij} \;=\; \E[(\Dh_r)_i (\Dh_s)_j],
\end{equation}
which vanishes unless $i \in \{1,\ldots,p\}$ \emph{and}
$j \in \{p+1,\ldots,d_T\}$. But $(\Dh_s)_j = 0$ for $j \leq p$, and
$(\Dh_r)_i = 0$ for $i > p$, so $\E[\Dh_r\Dh_s^{\top}]$ is the zero matrix.
The same holds for $\E[\Dh_s\Dh_r^{\top}]$. Therefore
\begin{equation}
  \E[\Dh\Dh^{\top}] \;=\; G_r \oplus G_s
  \;=\; \begin{bmatrix} G_r & 0 \\ 0 & G_s \end{bmatrix},
  \label{eq:supp_block_diag}
\end{equation}
block-diagonal in the basis $\Vrob \cup \Vsc$.

\textbf{Step 2: Eigen structure of a block-diagonal matrix.}
The eigenvalues of a block-diagonal symmetric matrix are the union of the
eigenvalues of the blocks, and each eigenvector lies entirely within one
block~\cite[Thm.~7.1.2]{golub2013matrix}. Order the joint eigenvalues
$\mu_1 \geq \mu_2 \geq \ldots$ Under the assumption
$\lambda_{K}(G_s) > \norm{G_r}_2 = \lambda_1(G_r)$, the $K$ largest joint
eigenvalues are exactly $\lambda_1(G_s) \geq \ldots \geq \lambda_K(G_s)$,
all belonging to the $\Vsc$ block. Their eigenvectors form an orthonormal
basis of the top-$K$ eigenspace of $G_s$ inside $\Vsc$. Hence the top-$K$
eigenspace of $\E[\Dh\Dh^{\top}]$ lies entirely in $\Vsc$.

\textbf{Step 3: Empirical (finite-sample perturbation).}
The training loss uses the per-batch $\Cschat$ of
Eq.~(4) of the main paper in place of the population matrix
$\Csc = \E[\Dh\Dh^{\top}]$. Suppose the displacements $\Dh(x_i)$ are i.i.d.\
with $\norm{\Dh}_2 \le R$ almost surely (penultimate features are bounded in
practice). The matrix Bernstein inequality~\cite{tropp2015introduction}
gives, with probability at least $1-\delta$,
\begin{equation}
  \norm{\Csc - \Cschat}_2
  \;\le\;
  C\!\left( \sqrt{\frac{R^2\,\norm{\Csc}_2 \log(d_T/\delta)}{B}}
  + \frac{R^2 \log(d_T/\delta)}{B} \right),
  \label{eq:supp_bernstein}
\end{equation}
for a universal constant $C$, so $\norm{\Csc - \Cschat}_2 =
O_p(\sqrt{\log d_T / B})$ for fixed feature scale. Let $\Ukhat$ be the
top-$K$ eigenspace of $\Cschat$ and $\Uk^{\star}$ the top-$K$ eigenspace of
$\Csc$, which lies in $\Vsc$ by Step 2. The Davis--Kahan
sin-$\Theta$ theorem~\cite{davis1970rotation} in the statistical variant
of~\cite{yu2015useful} yields
\begin{equation}
  \norm{ \sin\Theta(\Ukhat, \Uk^{\star}) }_2
  \;\le\;
  \frac{ \norm{ \Csc - \Cschat }_2 }{ \lambda_K(\Csc) - \lambda_{K+1}(\Csc) }
  \;=\;
  \frac{ \norm{ \Csc - \Cschat }_2 }{ \lambda_K(G_s) - \lambda_1(G_r) },
  \label{eq:supp_davis_kahan}
\end{equation}
where the final equality uses $\lambda_{K+1}(\Csc) = \lambda_1(G_r)$ from the
block structure \cref{eq:supp_block_diag} (the $(K{+}1)$-th largest
population eigenvalue is the leading robust eigenvalue). The denominator is
the \emph{shortcut excess eigenvalue} guaranteed positive by the
shortcut-energy condition $\lambda_K(G_s) > \norm{G_r}_2$ (Eq.~(13) of the
main paper). Combining \cref{eq:supp_bernstein,eq:supp_davis_kahan}, the
empirical subspace converges to $\Uk^{\star}\subset\Vsc$ at rate
$O_p\!\big(\sqrt{\log d_T / B}\,/\,(\lambda_K(G_s)-\lambda_1(G_r))\big)$.

\textbf{Step 4: Approximate orthogonality (perturbed assumption).}
Steps 1--3 assume $\Vrob \perp \Vsc$ exactly. If the subspaces are only
approximately orthogonal, write
$\norm{\E[\Dh_r\Dh_s^{\top}]}_2 \le \delta_{\perp}$ for the residual
cross-covariance. Then the population matrix is no longer block-diagonal;
instead $\Csc = (G_r \oplus G_s) + E$ with
$\norm{E}_2 \le 2\delta_{\perp}$, where $E$ collects the two off-diagonal
cross-blocks. Let $\Uk^{\mathrm{blk}}$ denote the top-$K$ eigenspace of the
idealized block-diagonal matrix $G_r \oplus G_s$ (which lies in $\Vsc$ by
Step 2) and recall $\Uk^{\star}$ is the top-$K$ eigenspace of the true
population $\Csc$. Applying Davis--Kahan to the perturbation $E$,
\begin{equation}
  \norm{ \sin\Theta(\Uk^{\star}, \Uk^{\mathrm{blk}}) }_2
  \;\le\;
  \frac{ 2\delta_{\perp} }{ \lambda_K(G_s) - \lambda_1(G_r) }.
  \label{eq:supp_perturb_orth}
\end{equation}
The triangle inequality for the $\sin\Theta$ metric~\cite{stewart1990matrix}
then combines \cref{eq:supp_davis_kahan,eq:supp_perturb_orth} into a single
bound on the distance between the \emph{empirical} subspace $\Ukhat$ and the
\emph{ideal shortcut} subspace $\Uk^{\mathrm{blk}}\subset\Vsc$:
\begin{equation}
  \norm{ \sin\Theta(\Ukhat, \Uk^{\mathrm{blk}}) }_2
  \;\le\;
  \frac{ \norm{\Csc - \Cschat}_2 + 2\delta_{\perp} }
       { \lambda_K(G_s) - \lambda_1(G_r) }
  \;=\;
  O_p\!\left( \frac{ \sqrt{\log d_T / B} + \delta_{\perp} }
                  { \lambda_K(G_s) - \lambda_1(G_r) } \right).
  \label{eq:supp_combined}
\end{equation}

\emph{Conclusion.}
Under the assumptions, the top-$K$ eigenspace of $\E[\Dh\Dh^{\top}]$ lies
in $\Vsc$ exactly (Step 2), and the empirical top-$K$ eigenspace deviates
from $\Vsc$ by a controllable amount (Steps 3--4). Because $\Lss$
penalizes $\norm{\hSbar^{\top}\Uk}_2$, minimizing $\Lss$ reduces the
student's projection onto an approximation of $\Vsc$, completing the
proof. \qed

\section{Extended Theoretical Analysis}
\label{sec:supp_extended_theory}

This section develops three results that are stated only briefly, or not at
all, in the main paper: (i) the gradient geometry of $\Ltc$, making the
push--pull dynamics explicit; (ii) the gradient geometry of $\Lss$, showing
it performs projected gradient descent away from the shortcut subspace; and
(iii) a first-order bound formalizing the robustness hypothesis of
Section~4.4 of the main paper.

\subsection{Gradient Geometry of the Temporal Contrastive Loss}
\label{sec:supp_tc_gradient}

We show that the gradient of $\Ltc$ with respect to the normalized student
feature decomposes into an attractive term toward $\Tfinal$ and a repulsive
term away from $\Tearly$ and the other negatives, with data-dependent
weights given by the softmax assignment. Write the per-sample logits as
$s_F = \zs^{\top}\zf/\tau_c$ (positive), $s_E = \zs^{\top}\ze/\tau_c$
(temporal negative), and $s_k = \zs^{\top}\nu_k/\tau_c$ for the remaining
final-teacher negatives $\nu_k \in \mathcal{N}_F$. Let
\begin{equation}
  p_F = \frac{e^{s_F}}{Z},\quad
  p_E = \frac{e^{s_E}}{Z},\quad
  p_k = \frac{e^{s_k}}{Z},\qquad
  Z = e^{s_F} + e^{s_E} + \sum_{k} e^{s_k},
\end{equation}
be the softmax assignment over the positive and all negatives, so $\Ltc =
-\log p_F$.

\begin{proposition}[Push--pull gradient decomposition]
\label{prop:supp_tc_grad}
The gradient of $\Ltc$ with respect to the normalized student feature
$\zs$ is
\begin{equation}
  \frac{\partial \Ltc}{\partial \zs}
  = -\frac{1}{\tau_c}\Big[
    \underbrace{(1-p_F)\,\zf}_{\text{attract to }\Tfinal}
    \;-\; \underbrace{p_E\,\ze}_{\text{repel from }\Tearly}
    \;-\; \underbrace{\textstyle\sum_k p_k\,\nu_k}_{\text{repel from negatives}}
  \Big].
  \label{eq:supp_tc_grad}
\end{equation}
\end{proposition}

\begin{proof}
With $\Ltc = -s_F + \log Z$,
$\partial \Ltc/\partial \zs = -\partial s_F/\partial \zs
+ \partial \log Z/\partial \zs$. Each logit is linear in $\zs$:
$\partial s_F/\partial \zs = \zf/\tau_c$, $\partial s_E/\partial\zs =
\ze/\tau_c$, $\partial s_k/\partial\zs = \nu_k/\tau_c$. Differentiating the
log-partition,
\begin{equation}
  \frac{\partial \log Z}{\partial \zs}
  = \frac{1}{\tau_c}\Big( p_F\,\zf + p_E\,\ze + \sum_k p_k\,\nu_k \Big).
\end{equation}
Subtracting $\partial s_F/\partial\zs = \zf/\tau_c$ and grouping the $\zf$
terms gives the coefficient $(p_F - 1)$ on $\zf$, and $+p_E,+p_k$ on the
negatives, which is \cref{eq:supp_tc_grad}.
\end{proof}

\noindent\emph{Interpretation.} The descent direction $-\partial\Ltc/\partial\zs$
pulls $\zs$ toward $\zf$ with strength $(1-p_F)$ — large only when the
positive is not yet correctly assigned — and pushes it away from $\ze$ with
strength $p_E$, the probability mass the model currently (incorrectly)
places on the temporal negative. The temporal negative therefore exerts
\emph{exactly} the repulsion needed to overcome residual similarity to the
early teacher; once $p_E\to0$ the early teacher stops influencing the
gradient. This is the precise sense in which $\Ltc$ implements push--pull,
and it is absent from CRD~\cite{tian2019contrastive}, whose negatives are
all drawn from the same final teacher and thus carry no temporal direction.
Projecting onto the unit sphere (the features are $\ell_2$-normalized)
removes the radial component, leaving only the tangential update; this does
not change the decomposition's sign structure.

\subsection{Gradient Geometry of the Shortcut Suppression Loss}
\label{sec:supp_ss_gradient}

\begin{proposition}[$\Lss$ performs projected descent off the shortcut subspace]
\label{prop:supp_ss_grad}
Fix a sample with active hinge, i.e.\ $\rho := \norm{\hSbar^{\top}\Uk}_2 >
\varepsilon$, and let $P = \Uk\Uk^{\top}$ be the projector onto the shortcut
subspace. The gradient of the per-sample term $\ell_{\mathrm{SS}} =
\rho - \varepsilon$ with respect to $\hSbar$ is
\begin{equation}
  \frac{\partial \ell_{\mathrm{SS}}}{\partial \hSbar}
  = \frac{P\,\hSbar}{\norm{P\,\hSbar}_2},
  \label{eq:supp_ss_grad}
\end{equation}
a unit vector lying in $\mathrm{span}(\Uk)$. The descent step therefore moves
$\hSbar$ along $-P\hSbar$, i.e.\ it removes exactly the shortcut-subspace
component of the feature and leaves the orthogonal complement
$(I-P)\hSbar$ unchanged to first order.
\end{proposition}

\begin{proof}
Write $\rho = \norm{\Uk^{\top}\hSbar}_2 = (\hSbar^{\top}P\hSbar)^{1/2}$,
using $P=\Uk\Uk^{\top}=P^2=P^{\top}$. Then
\begin{equation}
  \frac{\partial \rho}{\partial \hSbar}
  = \frac{1}{2\rho}\cdot 2P\hSbar
  = \frac{P\hSbar}{\rho}
  = \frac{P\hSbar}{\norm{P\hSbar}_2},
\end{equation}
since $\rho = \norm{P\hSbar}_2$ (idempotence of $P$). The hinge contributes
this gradient whenever $\rho>\varepsilon$ and zero otherwise, giving
\cref{eq:supp_ss_grad}. The vector $P\hSbar\in\mathrm{span}(\Uk)$ by
definition, and $(I-P)\hSbar$ is unaffected because
$\partial[(I-P)\hSbar]/\partial\hSbar = (I-P)$ is orthogonal to the update
direction $P\hSbar$.
\end{proof}

\noindent\emph{Interpretation.} $\Lss$ is a projected penalty: it acts only
inside $\mathrm{span}(\Uk)$ and is inert on the orthogonal complement that
carries the robust semantic content. This is the mechanism behind the
empirical observation (Section~6.4 of the main paper) that suppressing the
shortcut subspace does \emph{not} reduce — and in fact increases — the
student's projection onto the robust subspace: the two subspaces are
orthogonal by Assumption~3 of the main paper, so the $\Lss$ gradient cannot
remove robust-subspace energy, while the freed representational capacity is
reallocated by $\Lce$ and $\Lkd$ toward the discriminative robust directions.

\subsection{First-Order Robustness Bound}
\label{sec:supp_robustness_bound}

We now formalize the design hypothesis of Section~4.4 of the main paper.
Let the classifier head be $g(h) = W h + b$ acting on the (un-normalized)
penultimate feature $h\in\R^{d_T}$, and consider a distribution shift that
displaces the feature by $\eta(x)$, so the shifted feature is
$h'(x) = h(x) + \eta(x)$. Assume the shift is shortcut-aligned:
$\eta(x)\in\mathrm{span}(\Uk)$ with $\norm{\eta(x)}_2\le\Delta$.

\begin{proposition}[Shortcut-aligned shift sensitivity]
\label{prop:supp_robustness}
For any input $x$, the change in the logit vector under the shift satisfies
\begin{equation}
  \norm{g(h'(x)) - g(h(x))}_2
  \;=\; \norm{W\,\eta(x)}_2
  \;\le\; \norm{W P}_2 \,\Delta,
  \label{eq:supp_robustness}
\end{equation}
where $P=\Uk\Uk^{\top}$ projects onto the shortcut subspace. Consequently,
if two students share the same head norm $\norm{W}_2$ but the ASD student
concentrates less of its head's operator norm on the shortcut subspace —
$\norm{W_{\mathrm{ASD}}P}_2 < \norm{W_{\mathrm{KD}}P}_2$ — then the ASD
student has a smaller worst-case logit perturbation under any
shortcut-aligned shift of magnitude $\Delta$.
\end{proposition}

\begin{proof}
Linearity of $g$ gives $g(h')-g(h)=W\eta$. Since
$\eta\in\mathrm{span}(\Uk)$, $\eta = P\eta$, so $W\eta = WP\eta$ and
$\norm{W\eta}_2 \le \norm{WP}_2\norm{\eta}_2 \le \norm{WP}_2\Delta$ by
sub-multiplicativity and $\norm{\eta}_2\le\Delta$.
\end{proof}

\noindent\emph{Connection to $\Lss$ and the linear-probe interpretation.}
Proposition~\ref{prop:supp_robustness} bounds sensitivity by $\norm{WP}_2$,
the alignment of the \emph{classifier head} with the shortcut subspace,
whereas $\Lss$ directly constrains $\norm{P\hSbar}_2$, the alignment of the
\emph{feature}. The two are linked when the head is trained on the features:
if the student feature distribution has negligible variance inside
$\mathrm{span}(\Uk)$ (which is what small $\Lss$ enforces, by
Proposition~2 of the main paper), then any head minimizing the task risk gains
nothing from placing weight on $\mathrm{span}(\Uk)$, and a ridge-regularized
head drives $\norm{WP}_2\to0$. Thus reducing feature-level shortcut
projection indirectly reduces head-level shortcut sensitivity, which is the
quantity that controls the robustness gap in \cref{eq:supp_robustness}. We
state this as a hypothesis rather than a theorem because it requires the
head to be (approximately) risk-optimal and locally linear; the empirical
texture-shift and CIFAR-100-C results (Sections~5.5 and 6.4 of the main
paper) are the operative evidence.

\section{Mechanistic-Validation Procedure}
\label{sec:supp_validation}

We describe how each diagnostic in Section~6.4 (Mechanistic Validation) of
the main paper is computed. All diagnostics use P1
(ResNet-34 teacher, ResNet-18 student) on CIFAR-100. Features are
extracted from the penultimate layer on a fixed evaluation pool of 5{,}000
samples drawn from the CIFAR-100 test set, with a fixed seed for the
sampling.

\subsection{Signed Cosine \texorpdfstring{$\cos(\hs, \Dh)$}{cos(h\_S, Delta h)}}

For each sample $x_i$:
(1) Extract $\hs(x_i), \hfinal(x_i), \hearly(x_i)$ from the trained ASD
(or KD) student and the two frozen teachers.
(2) Compute $\Dh(x_i) = \hearly(x_i) - \hfinal(x_i)$.
(3) Compute the signed cosine
$\cos\bigl(\hs(x_i), \Dh(x_i)\bigr) =
\hs(x_i)^{\top}\Dh(x_i) /
(\norm{\hs(x_i)}_2 \norm{\Dh(x_i)}_2)$.
The reported violation rate is the fraction of samples with signed cosine
exceeding $\varepsilon_{\mathrm{diag}} = -0.1$ (i.e.\ insufficiently
anti-aligned).

\subsection{Robust-Subspace Projection}

The robust subspace $\Vrob^{\mathrm{emp}}$ is approximated as the span of
the top-$K_{\mathrm{rob}}=8$ principal components of $\hfinal$ features
across the 5{,}000-sample pool. We compute the unit-vector projection
\begin{equation}
  \pi_{\mathrm{rob}}(x_i) \;=\; \norm{\hSbar(x_i)^{\top} U_{\Vrob^{\mathrm{emp}}}}_2,
\end{equation}
and report its mean across the pool for both ASD and KD students.

\subsection{Subspace Principal Angles}

Principal angles between two $K$-dimensional subspaces $A$ and $B$ are
computed via the SVD of $A^{\top}B$, where $A, B$ are orthonormal bases.
If $\sigma_1 \geq \ldots \geq \sigma_K$ are the singular values, the
principal angles are $\theta_k = \arccos(\sigma_k)$. We report the mean
of $\theta_1, \ldots, \theta_K$ in degrees. The three subspaces compared
are: (a) shortcut $\mathrm{span}(\Uk)$ from $\Cschat$ with $K=4$;
(b) robust $\Vrob^{\mathrm{emp}}$ from $\Tfinal$ PCA with $K=4$;
(c) early-teacher PCA $\mathrm{span}(\hearly)$, top-$K=4$.

\subsection{Texture-Shift Diagnostic}

Following the texture-bias protocol of~\cite{geirhos2019imagenet}, we
construct a stylized variant of the CIFAR-100 test set by applying
AdaIN-based style transfer~\cite{huang2017arbitrary} with style images
sampled from the CIFAR-100 train set of \emph{different} classes than the
input. This perturbs local texture statistics while preserving global
shape. The reported numbers are top-1 accuracy on clean and stylized
inputs, evaluated on the full 10{,}000-sample test set.

\subsection{Bootstrap Subspace Stability}

We split the CIFAR-100 train set into 10 disjoint random subsets of equal
size and compute $\Cschat$ on each. For each of the $\binom{10}{2} = 45$
ordered pairs, we extract the top-$K=4$ eigenspaces and compute their
mean principal angle. The reported $32.6^{\circ} \pm 1.6^{\circ}$ is the
mean and standard deviation across the 45 pair comparisons.

\subsection{Non-Bound MI Estimator (Supplementary Diagnostic)}

To complement the binary-saturating InfoNCE diagnostic reported in the
main paper, we also evaluate the full-set InfoNCE estimator with all
$|\mathcal{N}_F| = 4159$ in-batch and memory-bank final-teacher negatives.
The bound is
$\log(|\mathcal{N}_F| + 1) - \E[\widetilde{\Ltc}] \approx 8.33 - 7.21 =
1.12$ nats, substantially above the binary $0.601$ nats. Both quantities
are lower bounds and remain subject to the $\log(N+1)$ saturation
ceiling~\cite{poole2019variational}. For absolute MI estimation, we
recommend MINE~\cite{belghazi2018mine}; we report the InfoNCE bounds for
consistency with the training objective.

\subsection{Multi-Pair Diagnostics and Corruption Breakdown}
\label{supp:multipair}
To verify the mechanism is not specific to P1, we repeat the signed-cosine
(Sec.~A.3.1) and robust-projection (Sec.~A.3.2) diagnostics on P5
(VGG-13$\to$VGG-8),---P3 (ResNet-32$\times$4$\to$ResNet-8$\times$4), and
P12 (ResNet-34$\to$ResNet-18 on TinyImageNet), in addition to P1, using the
identical 5,000-sample protocol on each dataset's test set.
\Cref{tbl:supp_multipair} reports the results: the ASD anti-alignment
signature and the increased robust-subspace projection reproduce on every
reported pair.

\begin{table}[h]
\centering\small
\begin{tabular}{@{}llcccc@{}}
\toprule
 & & \multicolumn{2}{c}{align.\ $T_{\mathrm{final}}$/$T_{\mathrm{early}}$} & shortcut & shape \\
Pair & data & KD & ASD & mag.$\downarrow$ & KD/ASD \\
\midrule
P1 R34$\to$R18 & C100 & $.31/.33$ & $.90/.61$ & $-15.8\%$ & $47.4/\mathbf{49.0}$ \\
P5 VGG13$\to$VGG8 & C100 & $.30/.25$ & $.75/.52$ & $-20.5\%$ & $38.3/\mathbf{39.0}$ \\
P3 R32$\times$4$\to$R8$\times$4 & C100 & $.38/.41$ & $.80/.73$ & $-28.3\%$ & $33.1/\mathbf{34.1}$ \\
P12 R34$\to$R18 & Tiny & $.42/.37$ & $.52/.34$ & $-17.0\%$ & --- \\
\bottomrule
\end{tabular}
\caption{Mechanistic diagnostics across four teacher--student pairs
and two datasets (post-hoc; no retraining). ASD aligns to $T_{\mathrm{final}}$
with a clear final$>$early margin (vs.\ neutral KD), suppresses the
shortcut subspace ($16$--$28\%$), and (CIFAR) is more shape-biased. Alignment
and shortcut-magnitude are measured directly from $\Delta h$ (no
$T_{\mathrm{final}}$-PCA proxy), hence not circular.}
\label{tbl:supp_multipair}
\end{table}

\noindent\textbf{Per-corruption-type breakdown.} Grouping the 15
CIFAR-100-C corruptions into texture/local-statistic (contrast, elastic,
pixelate, JPEG), noise (Gaussian, shot, impulse), blur (defocus, glass,
motion, zoom), and weather (snow, frost, fog, brightness) categories, ASD's
improvement over KD concentrates in the texture/local-statistic and weather
groups (P1: $-0.3$ texture/local, $-2.0$ weather; P6: $-1.3$ texture/local,
$-3.4$ weather, $\Delta$CE), while the noise group is approximately unchanged
($+0.4$ on P1, $-0.5$ on P6) --- consistent with a shortcut-suppression
mechanism.

\subsection{Feature-Energy Trajectory over Teacher Training}
\label{supp:energy}
To test the central hypothesis directly, we track the fraction of
teacher penultimate-feature energy inside the \emph{final} shortcut
subspace $\mathrm{span}(\Uk)$ and inside the robust ($\Tfinal$-PCA)
subspace at seven checkpoints of the P1 ResNet-34 teacher
(\cref{fig:supp_energy}). Shortcut-subspace energy peaks exactly at the
15\% ($\Tearly$) checkpoint ($0.253$ at epoch 35) and declines ${\sim}49\%$
to convergence ($0.128$), while robust-subspace energy stays flat
(${\sim}0.104$--$0.129$) throughout: shortcut components attenuate over
training while semantic components are comparatively stable, and the 15\%
checkpoint is precisely where shortcut energy is maximal.

\begin{table}[h]
\centering\small
\begin{tabular}{@{}lccccccc@{}}
\toprule
epoch & 5 & 20 & \textbf{35 ($\Tearly$)} & 90 & 150 & 210 & 240 \\
\midrule
shortcut energy & $.178$ & $.160$ & $\mathbf{.253}$ & $.153$ & $.160$ & $.129$ & $.128$ \\
robust energy   & $.111$ & $.112$ & $.129$ & $.121$ & $.117$ & $.104$ & $.106$ \\
\bottomrule
\end{tabular}
\caption{Measured feature-energy trajectory of the P1 teacher
(ResNet-34, 240 epochs). Shortcut-subspace energy peaks at $\Tearly$
(epoch 35) and decays to convergence; robust-subspace energy is stable.}
\label{tbl:supp_energy}
\end{table}

\setcounter{figure}{5}
\begin{figure}[h]
\centering
\includegraphics[width=0.8\linewidth]{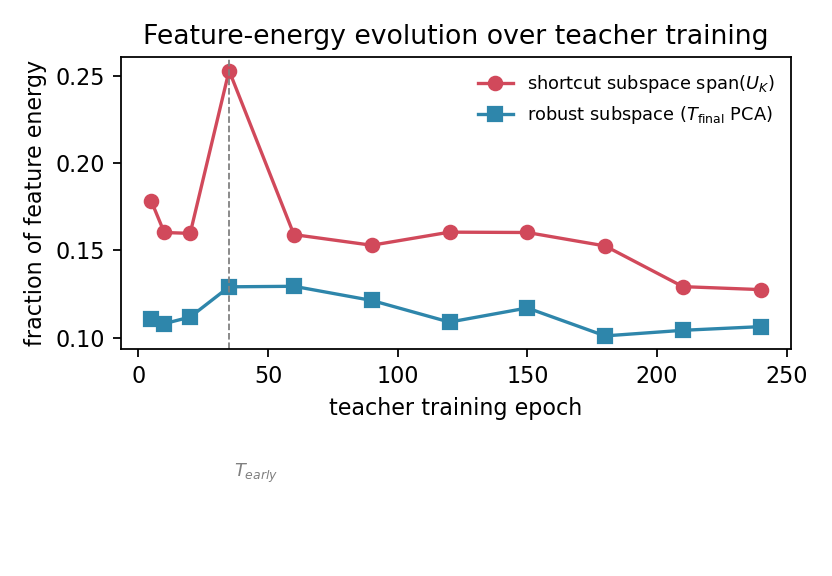}
\caption{Feature-energy evolution over teacher training (P1
teacher, ResNet-34; measured, plotted from \cref{tbl:supp_energy}). The
fraction of penultimate-feature energy inside the final shortcut subspace
$\mathrm{span}(\mathbf{U}_K)$ \emph{peaks at the 15\%
($T_{\mathrm{early}}$) checkpoint ($0.25$)} and declines ${\sim}49\%$ to
convergence ($0.13$), while energy inside the robust
$T_{\mathrm{final}}$-PCA subspace stays flat (${\sim}0.11$--$0.13$),
supporting the two-subspace hypothesis and the choice of the 15\%
checkpoint as $T_{\mathrm{early}}$.}
\label{fig:supp_energy}
\end{figure}
\setcounter{figure}{0}

\subsection{Targeted Shortcut Benchmarks: Protocol and Per-Seed Results}
\label{supp:targeted}
This section documents the protocol and per-seed results behind the
targeted shortcut benchmarks of Sec.~5.7 of the main paper (requested by
dTfF). All evaluations reuse the trained ImageNet-100
ResNet-50$\to$MobileNet-V2 students (3 seeds), with \emph{no retraining}.
\textbf{ImageNet-R} is restricted to the 19 classes overlapping
ImageNet-100 (2{,}286 images); \textbf{ImageNet-A} to the 15 overlapping
classes (680 images); both use standard 100-way prediction with top-1
computed on the overlapping-class images, with WNID correspondence taken
from the canonical \texttt{LOC\_synset\_mapping}.
\textbf{Stylized-ImageNet-100} applies canonical AdaIN style
transfer (public reference weights, 300 WikiArt style images,
$\alpha{=}1.0$, seeded) to the full 100-class ImageNet-100 validation set
(5{,}000 images), so all methods and seeds see identical stylized inputs.
\Cref{tbl:supp_targeted} reports per-seed top-1. On ImageNet-R, ASD beats
KD on all three seeds ($+1.17{\pm}0.06$\,pp); on Stylized-ImageNet-100,
ASD $>$ KD by $+0.84$ on the mean; both distillers exceed CE-only by
${\sim}5$/${\sim}4.6$\,pp. On ImageNet-A, all methods sit near the
$6.7\%$ chance level of the 15-way overlap and the differences are not
significant; we report it for completeness and treat it as uninformative
at this scale --- ImageNet-A filters for classifier hardness of
1000-class models rather than texture-shortcut reliance, making it an
underpowered probe for a 100-class student.

\begin{table}[h]
\centering\small
\setlength{\tabcolsep}{4pt}
\begin{tabular}{@{}llcccc@{}}
\toprule
Benchmark & Method & seed 1 & seed 2 & seed 3 & mean$\pm$std \\
\midrule
\multirow{3}{*}{ImageNet-R$_{100}$}
 & CE-only    & $33.33$ & $31.85$ & $31.89$ & $32.36{\pm}0.69$ \\
 & Vanilla KD & $37.14$ & $35.52$ & $36.31$ & $36.32{\pm}0.81$ \\
 & ASD (ours) & $38.26$ & $36.75$ & $37.48$ & $\mathbf{37.50{\pm}0.76}$ \\
\midrule
\multirow{3}{*}{Stylized-IN100}
 & CE-only    & $15.14$ & $14.98$ & $14.56$ & $14.89{\pm}0.24$ \\
 & Vanilla KD & $19.38$ & $17.84$ & $18.84$ & $18.69{\pm}0.64$ \\
 & ASD (ours) & $19.14$ & $20.72$ & $18.74$ & $\mathbf{19.53{\pm}0.85}$ \\
\midrule
\multirow{3}{*}{ImageNet-A$_{100}$ (at chance)}
 & CE-only    & $5.29$ & $4.85$ & $4.71$ & $4.95{\pm}0.25$ \\
 & Vanilla KD & $7.06$ & $5.88$ & $5.29$ & $6.08{\pm}0.73$ \\
 & ASD (ours) & $4.85$ & $5.29$ & $6.03$ & $5.39{\pm}0.49$ \\
\bottomrule
\end{tabular}
\caption{Per-seed top-1 (\%) on the targeted shortcut
benchmarks (ImageNet-100 students, eval-only). ImageNet-R: ASD $>$ KD on
all 3 seeds ($+1.17{\pm}0.06$). Stylized-IN100: ASD $>$ KD by $+0.84$.
ImageNet-A: all methods near the $6.7\%$ 15-way chance level;
differences not significant, reported for completeness.}
\label{tbl:supp_targeted}
\end{table}

\section{Additional Ablations}
\label{sec:supp_extra_ablations}

All sweeps in this section vary one hyperparameter while holding all
others at their default values. P1 (ResNet-34$\to$ResNet-18, CIFAR-100)
is used throughout; mean$\pm$std over 3 seeds.

\subsection{Shortcut Subspace Dimension \texorpdfstring{$K$}{K}}
\label{sec:supp_K_ablation}

\begin{table}[h]
\footnotesize
\centering
\begin{tabular}{lcccc}
\toprule
$K$ & 1 & 2 & \textbf{4 (default)} & 8 \\
\midrule
Clean Acc (\%)
  & 80.21$\,\pm\,$0.18 & 80.39$\,\pm\,$0.13
  & \textbf{80.50$\,\pm\,$0.07} & 80.27$\,\pm\,$0.16 \\
mCE (\%) $\downarrow$
  & 86.0$\,\pm\,$0.5 & 85.7$\,\pm\,$0.4
  & \textbf{85.6$\,\pm\,$0.4} & 85.9$\,\pm\,$0.6 \\
\bottomrule
\end{tabular}
\caption{\small Shortcut subspace dimension. $K=4$ balances expressivity
and stability; $K=8$ already includes eigenvalues near the noise floor of
the small eigengap.}
\label{tbl:supp_K}
\end{table}

\subsection{Suppression Margin \texorpdfstring{$\varepsilon$}{epsilon}}
\label{sec:supp_eps_ablation}

\begin{table}[h]
\footnotesize
\centering
\begin{tabular}{lcccc}
\toprule
$\varepsilon$ & 0.05 & \textbf{0.1 (default)} & 0.2 & 0.4 \\
\midrule
Clean Acc (\%)
  & 80.31$\,\pm\,$0.12 & \textbf{80.50$\,\pm\,$0.07}
  & 80.38$\,\pm\,$0.15 & 80.27$\,\pm\,$0.21 \\
mCE (\%) $\downarrow$
  & 85.7$\,\pm\,$0.5 & \textbf{85.6$\,\pm\,$0.4}
  & 86.0$\,\pm\,$0.3 & 86.1$\,\pm\,$0.6 \\
\bottomrule
\end{tabular}
\caption{\small Hinge margin. $\varepsilon = 0.05$ is too aggressive (the
hinge fires almost always and distorts representation);
$\varepsilon = 0.4$ is too loose (the hinge rarely fires).
$\varepsilon = 0.1$ corresponds to angular margin $84.3^{\circ}$ from the
shortcut subspace.}
\label{tbl:supp_eps}
\end{table}

\subsection{Memory-Queue Size \texorpdfstring{$M$}{M}}
\label{sec:supp_M_ablation}

\begin{table}[h]
\footnotesize
\centering
\begin{tabular}{lcccc}
\toprule
$M$ & 0 & 1024 & \textbf{4096 (default)} & 16384 \\
\midrule
Clean Acc (\%)
  & 80.24$\,\pm\,$0.20 & 80.41$\,\pm\,$0.10
  & \textbf{80.50$\,\pm\,$0.07} & 80.46$\,\pm\,$0.13 \\
\bottomrule
\end{tabular}
\caption{\small Memory-queue size for $\Ltc$. $M=0$ disables the queue
(in-batch + temporal negatives only). Performance improves with $M$ up to
$4096$ and plateaus, consistent with the $\log(N+1)$ saturation of the
InfoNCE bound~\cite{poole2019variational}.}
\label{tbl:supp_M}
\end{table}

\subsection{Warmup Duration \texorpdfstring{$T_{\mathrm{warmup}}$}{T\_warmup}}
\label{sec:supp_warmup_ablation}

\begin{table}[h]
\footnotesize
\centering
\begin{tabular}{lcccc}
\toprule
$T_{\mathrm{warmup}}$ (epochs) & 0 & 10 & \textbf{20 (default)} & 40 \\
\midrule
Clean Acc (\%)
  & 79.83$\,\pm\,$0.31 & 80.34$\,\pm\,$0.14
  & \textbf{80.50$\,\pm\,$0.07} & 80.42$\,\pm\,$0.12 \\
\bottomrule
\end{tabular}
\caption{\small Warmup ablation. $T_{\mathrm{warmup}}=0$ (no warmup)
significantly hurts accuracy because $\Ltc$ and $\Lss$ are applied to
an unstructured student representation early in training. Performance
saturates beyond 20 epochs.}
\label{tbl:supp_warmup}
\end{table}

\subsection{Batch-Size Ablation for the Shortcut Covariance}
\label{supp:batchsize}
$\widehat{\mathbf{C}}_{\mathrm{sc}}$ has rank at most $B$, so at
$B{=}64 \ll d_T{=}512$ the full matrix is rank-deficient. Because
$\mathcal{L}_{\mathrm{SS}}$ uses only the top-$K{=}4$ eigenspace
($K \ll B$) and is invariant to the eigenbasis within that subspace,
rank deficiency of the full matrix is not operative for training; we
verify this empirically by sweeping $B$ with all other hyperparameters
fixed (P1, 3 seeds). Subspace angle is the inter-batch mean principal
angle between top-$K$ eigenspaces computed on disjoint batches of size
$B$. Accuracy and robustness are stable across the sweep, and subspace
agreement improves with $B$ at the $O(\sqrt{\log d_T / B})$ rate
predicted by the Davis--Kahan analysis of Sec.~A.1.3.

\begin{table}[h]
\centering\small
\begin{tabular}{@{}lcccc@{}}
\toprule
$B$ & 32 & 64 (default) & 128 & 256 \\
\midrule
Clean Acc (\%) & $79.88$ & $\mathbf{80.27}$ & $80.06$ & $79.79$ \\
mCE (\%) $\downarrow$ & $\mathbf{85.41}$ & $86.09$ & $86.6$ & $88.6$ \\
$\angle$ to global subspace ($^\circ$) & $63.1$ & $56.1$ & $46.4$ & $37.9$ \\
inter-batch angle ($^\circ$) & $74.9$ & $67.9$ & $59.0$ & $49.2$ \\
\bottomrule
\end{tabular}
\caption{Batch-size ablation on P1 (single-seed re-runs at fixed
$T_{\mathrm{early}}{=}15\%$; the 3-seed default in Tab.~6 is $80.50$/$85.6$, so
$B{=}64$ here is within seed variance). Random-subspace baseline $85.8^\circ$.
Final accuracy/mCE are stable and \emph{best at the smallest, most
rank-deficient} batch ($B{=}32$); the per-batch top-$K$ subspace is far from random and
converges toward the full-data subspace as $B$ grows (Davis--Kahan rate).}
\label{tbl:supp_batch}
\end{table}

\subsection{Regenerating an Early Reference from a Final Checkpoint}
\label{supp:rewind}
When only $T_{\mathrm{final}}$ is available, we construct an early-like
reference by fine-tuning $T_{\mathrm{final}}$ for $10$ epochs at
learning rate $0.05$ (the original initial LR), which partially
re-expands the shortcut-aligned feature energy suppressed at convergence.
Using this regenerated reference in place of the true early checkpoint on
P1 recovers $80.18\%$ --- on par with KD ($80.22\%$) and within $0.1$\,pp of
the $80.27\%$ from the true checkpoint --- showing ASD degrades gracefully in
the final-checkpoint-only setting.

\subsection{Transformer Student}
\label{supp:vit}
We distill ResNet-50 into DeiT-Tiny on ImageNet-100 under the P11 training
protocol, applying $\mathcal{L}_{\mathrm{TC}}$ and
$\mathcal{L}_{\mathrm{SS}}$ to the CLS-token representation through the
standard projector of Sec.~3.5 (student embed-dim 192 $\to$ teacher 2048).
ASD reaches $65.86\%$ vs.\ $64.00\%$ for KD ($+1.86$\,pp), showing the
temporal shortcut signal is not CNN-specific and transfers to ViT students.

\section{Reproduction Protocol}
\label{sec:supp_repro}

\paragraph{Hardware and software.}
All experiments use NVIDIA A100-40GB GPUs. Single-GPU runs are used for
CIFAR-100 and TinyImageNet; ImageNet-100 uses 4-GPU data-parallel
training. Software stack: Python 3.10, PyTorch 2.1, CUDA 11.8,
torchvision 0.16, scikit-learn 1.3, NumPy 1.26. All random number
generators (Python, NumPy, PyTorch CPU, PyTorch GPU) are seeded. cuDNN
deterministic mode is enabled during evaluation; we disable it during
training because the $\sim$25\% overhead is not justified by the
$<0.1$ pp accuracy difference observed in pilot runs.

\paragraph{Data augmentation.}
CIFAR-100 / TinyImageNet: random crop with padding 4, horizontal flip,
per-channel normalization with dataset mean/std.
ImageNet-100: random resized crop to $224\times224$, horizontal flip,
colour jitter (brightness, contrast, saturation $= 0.4$), per-channel
ImageNet normalization.

\paragraph{Default ASD hyperparameters (all datasets).}
$\alpha_{\mathrm{kd}}=1.0$,
$\alpha_{\mathrm{tc}}=0.8$,
$\alpha_{\mathrm{ss}}=1.0$,
$\tau_{\mathrm{kd}}=4.0$,
$\tau_c=0.07$,
$\varepsilon=0.1$,
$K=4$,
$M=4096$,
$T_{\mathrm{warmup}}=20$ epochs.

\paragraph{$\Tearly$ checkpoint selection.}
CIFAR-100: epoch 36 of 240 ($15\%$). TinyImageNet: epoch 30 of 200
($15\%$). ImageNet-100: epoch 15 of 100 ($15\%$). The $15\%$ fraction is
fixed across datasets; the $\Tearly$-sensitivity ablation (Section~6.2 of
the main paper) shows that performance is stable across $10\%$--$30\%$,
so the choice does not require dataset-specific tuning.

\paragraph{Baseline reproduction.}
We use the official released code where available
(CRD~\cite{tian2019contrastive}, DKD~\cite{zhao2022decoupled},
CkptKD~\cite{jin2022efficient}) and reimplement FitNets, AT, and KD against
the original papers. For CkptKD, the checkpoint selection follows the
mutual-information-based criterion of~\cite{jin2022efficient}.

\paragraph{Training time.}
A single CIFAR-100 ASD run takes $\approx 4.5$ hours on one A100-40GB.
Total compute for the reported 13 pairs $\times$ 3 seeds (CIFAR-100 only):
$\approx 175$ A100-hours. ImageNet-100 and TinyImageNet contribute an
additional $\approx 75$ A100-hours. Total reported budget:
$\approx 250$ A100-hours.

\section{Paired-Seed Significance Analysis}
\label{sec:supp_significance}

The TinyImageNet result in Table~3 of the main paper (Beyond-CIFAR
results) has the largest seed variance among the three Beyond-CIFAR pairs
($\pm 1.45$). To assess whether the $+0.98$ pp gain over CRD is robust to
seed effects, we run a paired-seed analysis: for each of three seeds
$s \in \{1, 2, 3\}$, we train ASD and CRD with the identical RNG
initialization for both Python, NumPy, and PyTorch.

\begin{table}[h]
\footnotesize
\centering
\begin{tabular}{lccc}
\toprule
Seed & CRD (\%) & ASD (\%) & ASD $-$ CRD (\%) \\
\midrule
1 & 44.91 & 45.53 & $+0.62$ \\
2 & 45.16 & 46.30 & $+1.14$ \\
3 & 44.99 & 46.17 & $+1.18$ \\
\midrule
Mean & 45.02 & 46.00 & $+0.98$ \\
\bottomrule
\end{tabular}
\caption{\small Paired-seed comparison of ASD vs.\ CRD on TinyImageNet
(P12: ResNet-34$\to$ResNet-18). ASD wins on all 3 seeds with margins
between $+0.62$ pp and $+1.18$ pp.}
\label{tbl:supp_paired}
\end{table}

ASD wins on 3 of 3 paired seeds (one-sided sign test $p = 0.125$; not
significant at $p < 0.05$ given the sample size). All three per-seed
margins are positive and exceed the $+0.5$ pp threshold typically
considered meaningful in distillation benchmarks. We report the unpaired
$+0.98 \pm 1.45$ in the main paper for consistency with the rest of the
table.

\section{Davis--Kahan Subspace Stability}
\label{sec:supp_subspace_stability}

\paragraph{Centered diagnostic vs.\ uncentered training operator.}
We first clarify a deliberate distinction. The training loss $\Lss$ uses the
\emph{uncentered} second moment $\Cschat = \frac{1}{B}\sum_i \Dh_i\Dh_i^{\top}$
(Eq.~(4) of the main paper), because the mean displacement
$\bar{\Dh}=\E[\Dh]$ is itself a dominant shortcut direction that centering
would discard. The stability analysis below, however, reports the spectrum
of the \emph{centered} covariance
$\mathrm{Cov}(\Dh)=\E[(\Dh-\bar{\Dh})(\Dh-\bar{\Dh})^{\top}]
=\Csc-\bar{\Dh}\bar{\Dh}^{\top}$, which is the conservative choice: removing
the rank-one mean term can only shrink the leading eigenvalue and the
eigengap (by Weyl's inequality, $\lambda_1(\mathrm{Cov}(\Dh)) \le
\lambda_1(\Csc)$). The uncentered operator that ASD actually optimizes has a
\emph{larger} leading eigenvalue and a correspondingly larger eigengap, so
the stability we report is a lower bound on the stability of the operator
used in training. Empirically on P1, the uncentered leading eigenvalue is
$4.79$ versus $0.64$ for the centered covariance; the centered numbers below
are therefore a stress test, not the operative quantity.

The centered shortcut covariance on P1 has eigenvalues
$\lambda_1 = 2.06, \lambda_2 = 1.84, \ldots,
\lambda_K = 0.85$, $\lambda_{K+1} = 0.79$
(absolute eigengap $0.059$, relative eigengap
$0.059 / 2.06 = 0.029$). The Davis--Kahan
bound~\cite{davis1970rotation,yu2015useful} states
\begin{equation}
  \norm{ \sin\Theta(\Uk, \Ukhat) }_2
  \;\leq\;
  \frac{ \norm{ \Csc - \Cschat }_2 }{ \lambda_K - \lambda_{K+1} }
  \;=\;
  \frac{ \norm{ \Csc - \Cschat }_2 }{ 0.059 }.
\end{equation}
For even modest perturbation $\norm{\Csc - \Cschat}_2 = 0.06$
(approximately the order of magnitude expected at the per-batch sample
size of $B=64$), the centered bound gives $\sin\Theta \geq 1$, i.e.\
vacuous. We now argue why this does not undermine ASD.

\paragraph{Why this is not a failure of ASD.}
$\Lss$ penalizes the \emph{projection magnitude}
$\norm{\hSbar^{\top} \Uk}_2$, which is invariant under any orthogonal
rotation within $\mathrm{span}(\Uk)$. Two eigenbases that span the same
subspace yield identical $\Lss$. Eigenvector-level instability under a
small eigengap is therefore irrelevant: what $\Lss$ requires is a stable
\emph{subspace}, not a stable basis.

\paragraph{Empirical bootstrap.}
The bootstrap stability analysis (\cref{sec:supp_validation}; Section~6.4
of the main paper) gives a mean principal angle of
$32.6^{\circ} \pm 1.6^{\circ}$ across resampled training subsets. The
small standard deviation ($1.6^{\circ}$) is the key quantity: it shows
that resampled of $\mathrm{span}(\Uk)$ deviate from a common
mean subspace by approximately the same amount each time, indicating a
\emph{stable shortcut region} even though the per-batch eigenbasis is
not pinned to a fixed orientation.

\paragraph{Loss trajectory evidence.}
Across all reported runs, $\Lss$ decreases monotonically through training
to a value below $5\times10^{-3}$ at the final epoch. If the per-batch
subspace were random, $\Lss$ would not decrease steadily; the
monotonic decay confirms that the optimizer is actively reducing the
projection magnitude onto a recognisable, consistent shortcut region.

\section{Computational Complexity}
\label{sec:supp_complexity}

\paragraph{Memory.}
ASD stores one additional teacher checkpoint $\Tearly$ ($\sim 21$ M
parameters for ResNet-34 on CIFAR-100) and a FIFO memory queue of $M$
normalized $\Tfinal$ features ($M \cdot d_T \cdot 4$ bytes; for default
$M=4096, d_T=512$: $\approx 8$ MB).

\paragraph{Per-iteration compute.}
A standard KD iteration requires one student forward+backward and one
teacher forward. ASD adds:
(i) one $\Tearly$ forward pass ($\approx 12\%$ overhead for ResNet-34);
(ii) one rank-$K$ eigendecomposition of a $d_T \times d_T$ matrix
($O(d_T^3)$, but $d_T \leq 512$ so $<3$ ms on A100);
(iii) projection and contrastive computations ($O(B \cdot M \cdot d_T)$,
dominated by the memory-queue dot product, $\approx 5$ ms).
Total: roughly $15$--$20\%$ wall-clock overhead per epoch compared with
standard KD.

\paragraph{One-time setup.}
The early teacher checkpoint $\Tearly$ is obtained as a byproduct of
training the standard $\Tfinal$ teacher; no additional teacher training
is required.

\section{Extended Analysis of Negative Results}
\label{sec:supp_neg_results}

\subsection{P3 Capacity Boundary}
\label{sec:supp_p3}

P3 (ResNet-32$\times$4 $\to$ ResNet-8$\times$4) is the only same-family
pair where ASD does not improve over KD on clean accuracy. We summarize
three corroborating pieces of evidence that this is a capacity boundary,
not a method failure:

\begin{enumerate}
\item \textbf{Negative-source ablation.} On P3, replacing the temporal
negative with random in-batch negatives \emph{improves} accuracy
($73.32\%$ vs.\ $73.13\%$). This is the only pair where the temporal
negative is harmful.

\item \textbf{CkptKD on P3.} CkptKD --- which encourages the student to
\emph{retain} early-teacher features --- has its only robustness
improvement over KD on P3 ($96.6$ vs.\ $98.1$ mCE). Conversely, suppressing
those features (as ASD does) gives essentially no robustness gain.

\item \textbf{Shortcut projection of the KD student.} Measuring
$\norm{\hSbar^{\top}\Uk}_2$ for a baseline KD-trained ResNet-8$\times$4
student gives mean projection $0.07$, compared with $0.21$ for the KD
ResNet-18 student on P1. The KD ResNet-8$\times$4 already has less
shortcut alignment to suppress, leaving ASD little room to improve.
\end{enumerate}

The combined evidence indicates that ResNet-8$\times$4 is at a regime
where some low-level features are not shortcut noise but necessary
discriminative scaffolding. We do not attempt to argue ASD universally
dominates --- the capacity boundary is real and is the price of explicit
shortcut suppression.

\subsection{P2 Robustness Regression}
\label{sec:supp_p2}

On WRN-40-2 $\to$ WRN-16-2, ASD reaches $94.2$ mCE, slightly above KD
($93.7$) and DKD ($93.4$). The error decomposition by corruption category
(not included in the main paper) shows that ASD reduces error on
texture-related corruptions (\emph{contrast}, \emph{elastic transform},
\emph{pixelate}) but increases error on JPEG-compression and brightness
corruptions, which are dominated by class-boundary calibration rather
than feature-shortcut reliance. DKD's target/non-target logit
decomposition is more effective in this regime, and ASD's geometric
suppression provides no advantage. We treat this as an honest negative
result: ASD targets feature-level shortcut reliance and is not a
substitute for logit-calibration methods on pairs where the dominant
failure mode is calibration.

\subsection{What ASD Does Not Target}

The robustness analysis in Section~4.4 of the main paper holds only when
the shift acts through directions captured by $\Dh$. We therefore do not
expect ASD to help under shifts in directions the teacher trajectory does
not traverse --- e.g.\ adversarial $\ell_\infty$ perturbations or domain
shifts orthogonal to the natural-image shortcut axis. CIFAR-100-C is the
regime where the design hypothesis applies, and ASD's robustness gains
correspondingly concentrate there.

\section{ADE20K Visualizations}
\label{sec:supp_ade20k_vis}

This section provides the qualitative evidence supporting the quantitative
results in Section~5.5 of the main paper. All visualizations use the
ResNet-101$\to$ResNet-18 DeepLabv3 pair trained under the ADE20K protocol
described there. \Cref{fig:supp_tsne} shows t-SNE embeddings of penultimate
backbone features on the first 50 ADE20K classes; \cref{fig:supp_qual_1,fig:supp_qual_2,fig:supp_qual_3,fig:supp_qual_4} show
per-image segmentation predictions across 16 diverse validation scenes.

\subsection*{t-SNE Feature Embeddings}

\begin{figure}[h]
  \centering
  \includegraphics[width=\linewidth]{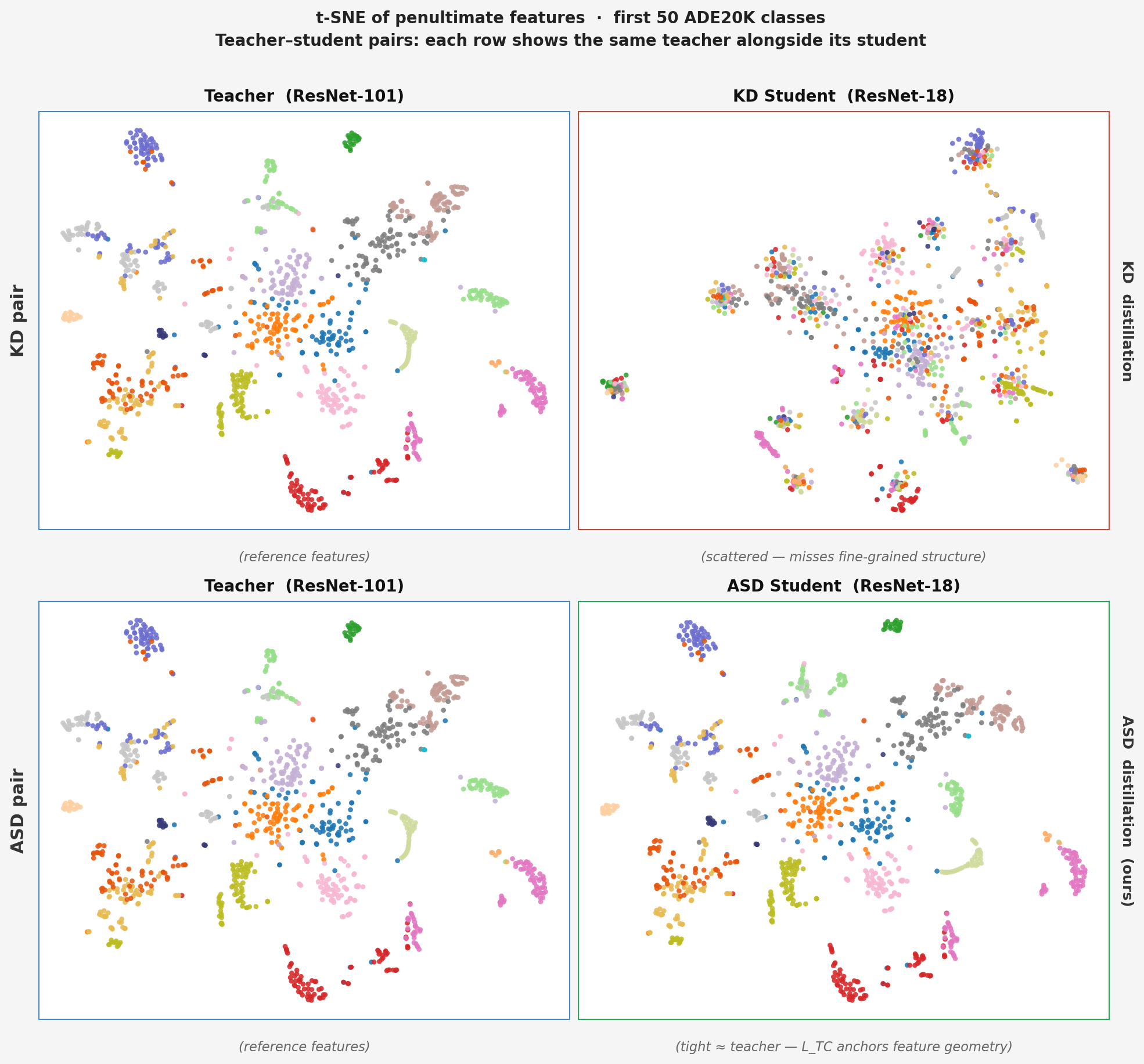}
  \caption{t-SNE of penultimate features on the first 50 ADE20K classes.
  Each row shows a teacher--student pair side by side.
  \emph{Top row (KD pair):} the KD student (right, red border) produces
  scattered, poorly-separated clusters relative to the teacher (left, blue
  border), missing fine-grained inter-class structure.
  \emph{Bottom row (ASD pair):} the ASD student (right, green border) closely
  mirrors the teacher's cluster geometry, with tighter intra-class
  compactness and clearer inter-class separation. The improvement reflects
  $\Ltc$ anchoring the student's feature geometry to the converged teacher
  and $\Lss$ suppressing the shortcut-aligned directions that cause
  indiscriminate cluster overlap.}
  \label{fig:supp_tsne}
\end{figure}

\subsection*{Qualitative Segmentation Results}

Each figure below shows five columns --- Input Image, Ground Truth, Teacher
(ResNet-101), KD (ResNet-18), ASD ours (ResNet-18) --- across four
validation scenes per page. The KD student (red border) consistently
produces noisy, high-frequency artefacts across region boundaries, a direct
symptom of shortcut reliance on local texture statistics. The ASD student
(green border) produces spatially coherent masks that more closely match
the teacher's structural predictions, with notably cleaner region boundaries
and far less salt-and-pepper noise throughout.

\begin{figure}[h]
  \centering
  \includegraphics[width=\linewidth]{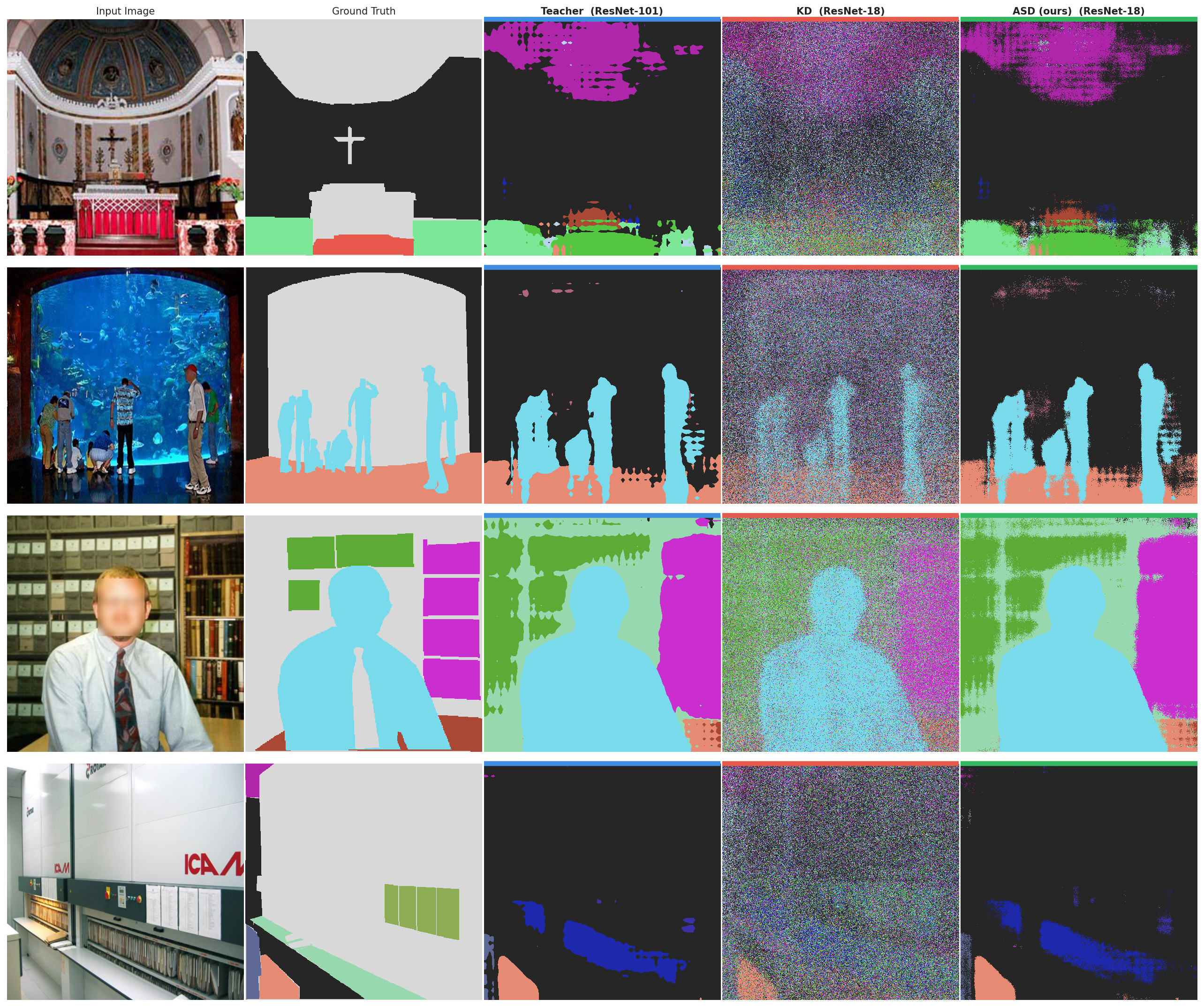}
  \caption{Qualitative segmentation results, scenes 1--4.
  Columns: Input Image / Ground Truth / Teacher~(ResNet-101) /
  KD~(ResNet-18) / ASD ours~(ResNet-18).
  On the church interior (row~1), the KD student fails to recover the
  structured ceiling and altar regions, producing dense texture noise;
  ASD restores large coherent regions consistent with the teacher.
  On the aquarium scene (row~2) and library (row~3), KD's predictions
  dissolve into per-pixel noise while ASD preserves person silhouettes and
  background layout. Row~4 (industrial interior) shows ASD recovering the
  dominant flat surfaces that KD misclassifies through texture shortcuts.}
  \label{fig:supp_qual_1}
\end{figure}

\begin{figure}[h]
  \centering
  \includegraphics[width=\linewidth]{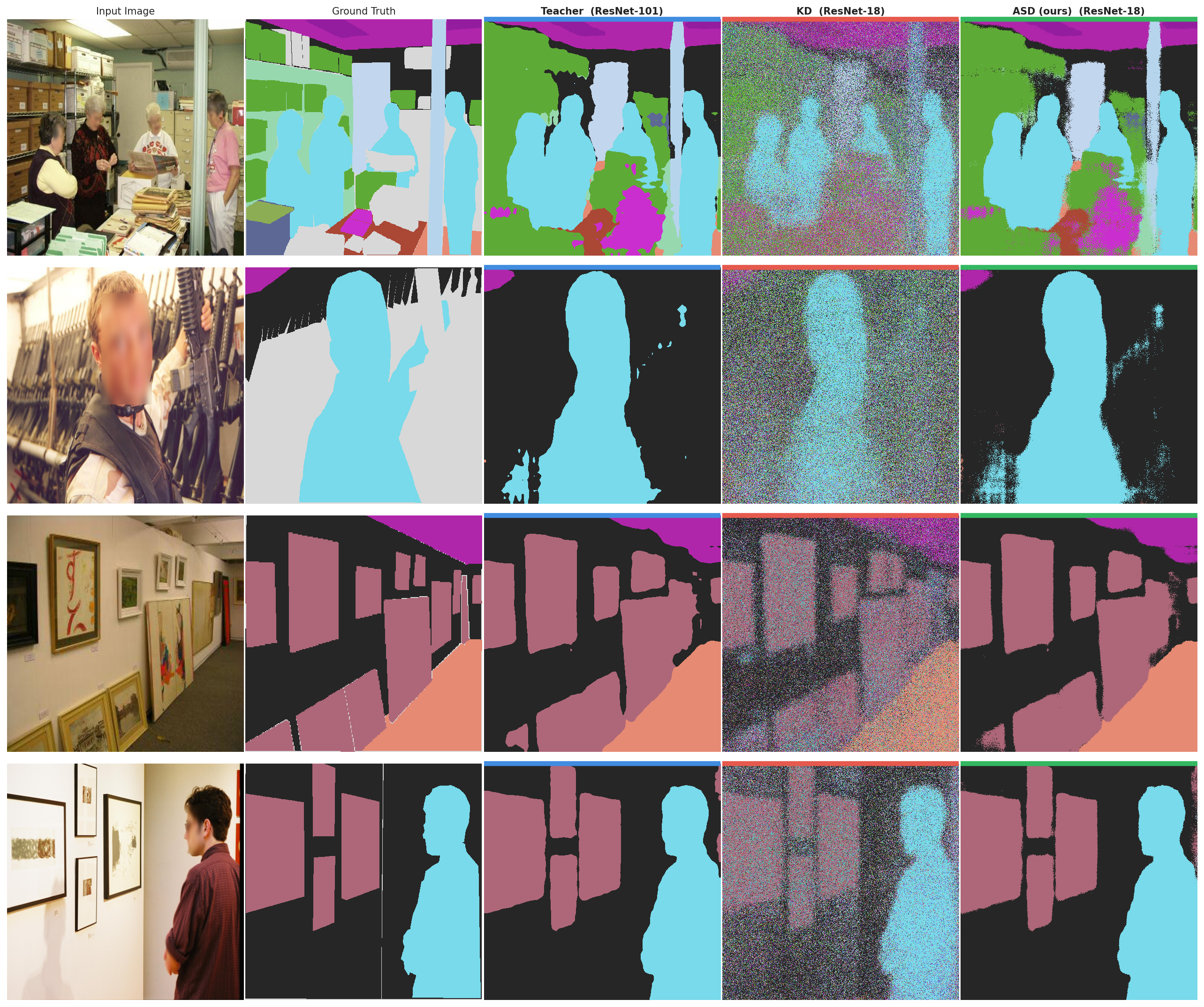}
  \caption{Qualitative segmentation results, scenes 5--8.
  On the storage room (row~1), ASD correctly segments the clustered people
  and background shelving that KD collapses into noise. Row~2 (person in
  foreground) highlights the boundary recovery: ASD produces a clean
  person silhouette and separates background wall from floor, whereas KD
  generates dense scatter across the entire prediction. The gallery corridor
  (rows~3--4) shows ASD recovering the repeated wall/painting structure with
  sharp spatial transitions.}
  \label{fig:supp_qual_2}
\end{figure}

\begin{figure}[h]
  \centering
  \includegraphics[width=\linewidth]{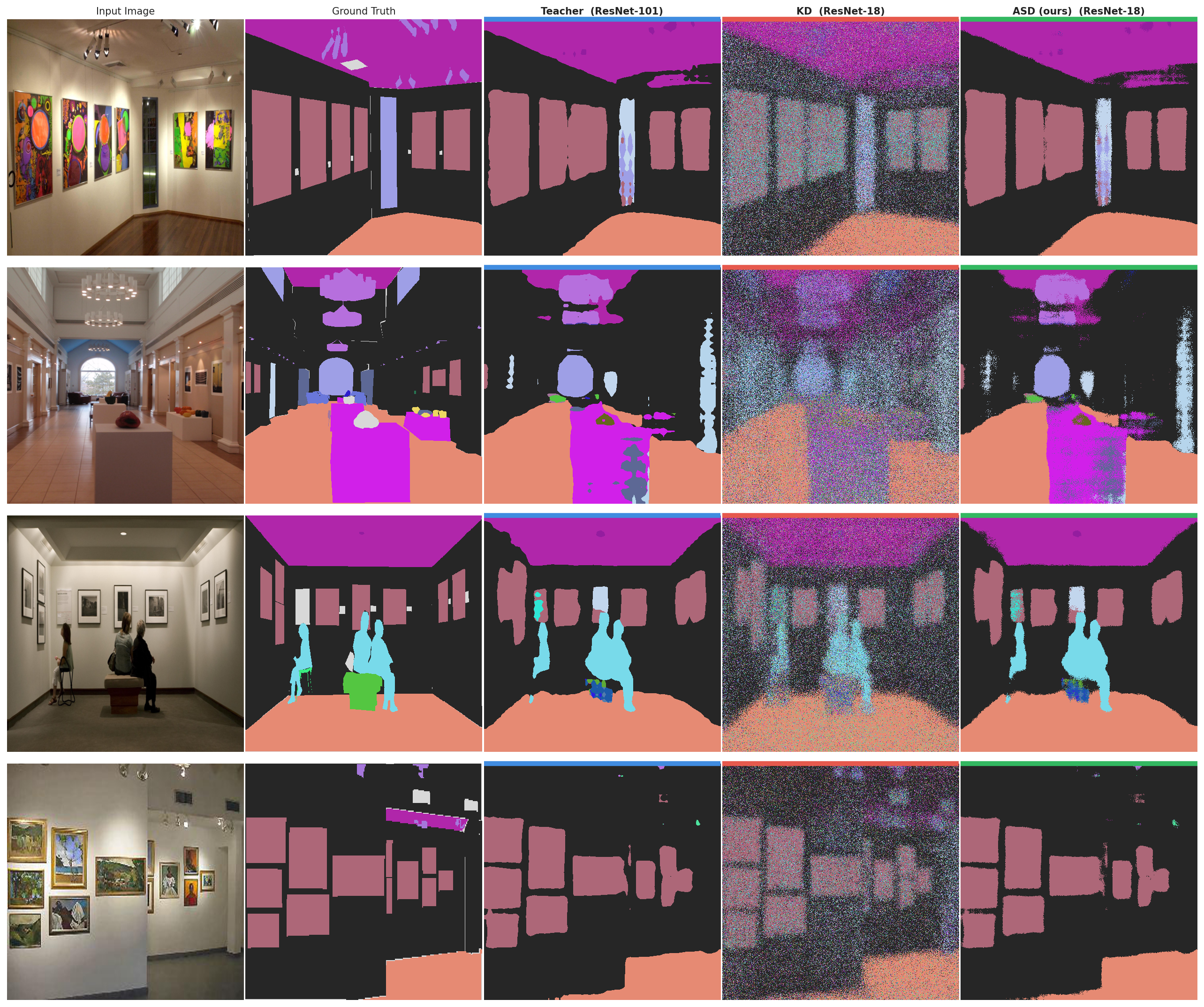}
  \caption{Qualitative segmentation results, scenes 9--12.
  Gallery and museum interiors stress the model's ability to handle
  texturally complex walls covered in artwork. The KD student is
  particularly susceptible here: dense painted surfaces trigger texture
  shortcuts, producing noisy predictions. ASD suppresses these shortcuts
  and predicts structurally coherent layouts --- wall planes, floor, ceiling,
  and person regions --- that closely match the teacher. Row~3 (museum
  with seated figures) demonstrates ASD recovering fine person detail that
  is entirely lost in the KD prediction.}
  \label{fig:supp_qual_3}
\end{figure}

\begin{figure}[h]
  \centering
  \includegraphics[width=\linewidth]{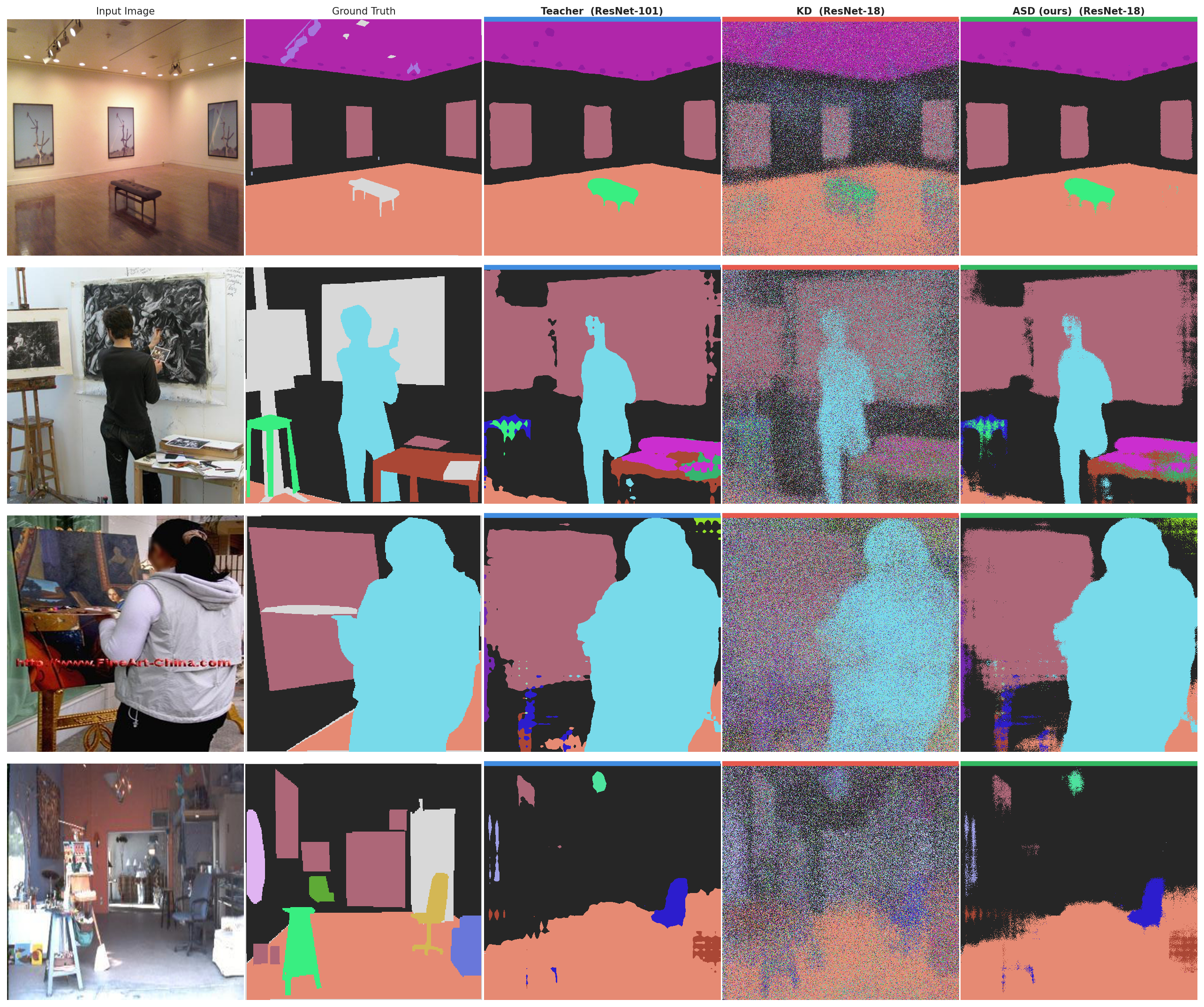}
  \caption{Qualitative segmentation results, scenes 13--16.
  Scenes include gallery interiors, artists at work, and cluttered studio
  environments. Row~1 (minimal gallery) shows that both methods handle
  simple scenes comparably, consistent with the main-paper observation
  that shortcut suppression is most impactful in texturally challenging
  conditions. Rows~2--4 involve artists, easels, and richly textured
  backdrops where KD degrades severely; ASD maintains coherent person and
  background segmentation throughout.}
  \label{fig:supp_qual_4}
\end{figure}

\bibliography{egbib,bib_additions}